\documentclass[twoside]{article}
\usepackage[accepted]{aistats2015}

\pdfoutput=1
\usepackage{color}
 \usepackage{tikz} 
 \usetikzlibrary{arrows,automata}
\usepackage{times,amsmath,amssymb,amsfonts,epsfig,graphicx, mathrsfs}
\usepackage{amsthm}  
 \usepackage{dsfont}
 \newcommand{\denselist}{\itemsep 0pt\topsep-6pt\partopsep-6pt}
\newtheorem{definition}{Definition}

\newtheorem{proposition}{Proposition}
\newtheorem{lemma}{Lemma}

\newtheorem{remark}{Remark}

\newcommand{\dom}{\mathrm{dom}}
\newcommand{\sign}{\mathrm{sign}}
\newcommand{\TU}{\mathrm{TU}}

\newcommand{\supp}{\mathrm{supp}}

\newcommand{\R}{\mathbb{R}}
\newcommand{\Tb}{\boldsymbol{T}}
\newcommand{\T}{\mathcal{T}}
\newcommand{\D}{\mathcal{D}}
\newcommand{\Db}{\boldsymbol{D}}
\newcommand{\M}{\boldsymbol{M}}

\newcommand{\G}{\mathcal{G}}

\newcommand{\sss}{\mathcal{S}}

\newcommand{\Pt}{\mathcal{P}}

\newcommand{\A}{\boldsymbol{A}}
\newcommand{\E}{\boldsymbol{E}}
\newcommand{\Et}{\mathcal{E}}
\newcommand{\e}{\boldsymbol{e}}
\newcommand{\Hb}{\boldsymbol{H}}
\newcommand{\I}{\boldsymbol{I}}
\newcommand{\x}{\boldsymbol x}

\newcommand{\y}{\boldsymbol y}
\newcommand{\s}{\boldsymbol s}
\newcommand{\z}{\boldsymbol z}
\newcommand{\w}{\boldsymbol w}

\newcommand{\bb}{\boldsymbol b}
\newcommand{\cb}{\boldsymbol{c}}
\newcommand{\db}{\boldsymbol{d}}
\newcommand{\GG}{\mathfrak{G}}
\newcommand{\omegabf}{\boldsymbol \omega}
\newcommand{\betabf}{\boldsymbol \beta}
\newcommand{\Bbd}{\boldsymbol{B}}

\newcommand{\1}{\mathds{1}}

\newcommand{\colvec}[1]{\begin{bmatrix}#1\end{bmatrix}}
\newcommand{\colcst}[1]{\substack{#1}}

\newtheorem*{prop}{Proposition}
\begin{document}


%

%

\twocolumn[

\aistatstitle{A totally unimodular view of structured sparsity}

\aistatsauthor{ Marwa El Halabi \And Volkan Cevher}

\aistatsaddress{ LIONS, EPFL \And LIONS, EPFL } ]

\vspace{-10pt}
\begin{abstract}
\vspace{-15pt}
This paper describes a simple framework for structured sparse recovery based on convex optimization. We show that many structured sparsity models can be naturally represented by linear matrix inequalities on the support of the unknown parameters, where the constraint matrix has a totally unimodular (TU) structure. 
For such structured models, tight convex relaxations can be obtained in polynomial time via linear programming. Our modeling framework unifies the prevalent structured sparsity norms in the literature, introduces new interesting ones, and renders their tightness and tractability arguments transparent.
\end{abstract}

\vspace{-5mm}
\section{Introduction}\vspace{-3mm}
\label{sec:intro}
Many important machine learning problems reduce to extracting parameters from dimensionality-reduced and potentially noisy data \cite{wainwright2014structured}. The most common data model in this setting takes the familiar linear form
\vspace{-5pt}
\begin{equation}\label{eq: obs}
\bb = \A\left(\x^\natural\right) +\w,   
\end{equation} \vspace{-20pt}

where $\x^\natural \in \mathbb{R}^p$ is the unknown parameter that we seek, the linear operator $\A:\mathbb{R}^{p} \rightarrow \mathbb{R}^{n}$ compresses $\x^\natural$ from $p$  to $n\ll p$ dimensions, and $\w \in \mathbb{R}^n$ models noise. In a typical scenario, we only know $\A$ and $\bb$ in \eqref{eq: obs}. 

In the absence of additional assumptions, it is impossible to reliably learn $\x^\natural$ when $n\ll p$, since  $\A$ has a nontrivial nullspace. Hence, we must exploit application-specific knowledge on $\x^\natural$. This knowledge often imposes $\x^\natural$ to be \emph{simple}, e.g., well-approximated by a sparse set of coefficients that obey  domain structure. Indeed, structured sparse parameters frequently appear in machine learning, signal processing, and theoretical computer science, and have broader generalizations including structure on matrix-valued $\x^\natural$ based on its rank.   

This paper relies on the convex optimization perspective for structured sparse recovery, which offers a rich set of analysis tools for establishing sample complexity for recovery and algorithmic tools for obtaining numerical solutions \cite{chandrasekaran2012convex}. 
To describe our approach, we focus on the proto-problem
\begin{equation}\label{eq: proto}
  \text{Find the {\it simplest}~} {\x} \text{~ subject to {\it structure} and {\it data}}. \vspace{-7pt}
\end{equation} 

Fortunately, we often have immediate access to convex functions that encode the information within \emph{data} (e.g., $\|\A\x-\bb\| \le \sigma$ for some $\sigma \in \mathbb{R}_+$) in \eqref{eq: proto}. 

However, choosing convex functions that jointly address \emph{simplicity} and \emph{structure} in the proto-problem requires some effort, since their natural descriptions are inherently combinatorial \cite{baraniuk2010model,bach2011learning,nirav2013tractability,huang2011learning}. For instance, sparsity (i.e., the number of nonzero coefficients) of $\x^\natural$ subject to discrete restrictions on its support (i.e., the locations of the sparse coefficients) initiates many of the structured sparsity problems. Unsurprisingly, there is a whole host of useful convex functions in the literature that induce sparsity with the desiderata in this setting (cf., \cite{bach2011learning} for a review).  The challenge resides in finding computationally tractable convex surrogates that tightly captures the combinatorial models. 

To this end, this paper introduces a combinatorial sparse modeling framework that simultaneously addresses both tractability and tightness issues that arise as a result of convex relaxation.  In retrospect, our key idea is quite simple and closely follows the recipe in \cite{bach2010structured, obozinski2012convex}, but with some new twists: We first summarize the discrete constraints that encode structure as linear inequalities. We then identify whether the structural constraint matrix is totally unimodular (TU), which can be verified in polynomial-time \cite{truemper1982alpha}. We then investigate classical discrete notions of simplicity and derive the Fenchel biconjugate of the combinatorial descriptions to obtain the convex relaxations for \eqref{eq: proto}. 

We illustrate how  TU descriptions of simplicity and structure make many popular norms in the literature transparent, such as the (latent) group norm, hierarchical norms, and norms that promote \emph{exclusivity}. Moreover, we show that TU descriptions of sparsity structures support tight convex relaxations and polynomial-time solution complexity for \eqref{eq: proto}. Our tightness result is a direct corollary of the fact that TU inequalities result in an integral constraint polyhedron where we can optimize linear costs exactly by convex relaxation (cf., Lemma \ref{prop:integral_polyhedra}).

Our specific contributions are summarized as follows. We propose a new generative framework to construct structure-inducing convex programs (which are not necessarily structure-inducing norms). Our results complement the structured norms perspective via submodular modeling, and go beyond it by deriving tight convex norms for non-submodular models. We also derive novel theoretical results using our modeling framework. For instance, the latent group lasso norm is the tightest convexification of the group $\ell_0$-norm \cite{baldassarre2013group}; Hierarchical group lasso is the tightest convexification of the sparse rooted connected tree model \cite{baldassarre2013group}; Sparse-group lasso leads to combinatorial descriptions that are provably not totally unimodular; Exclusive lasso norm is tight even for overlapping groups. 

\vspace{-7pt}
\section{Preliminaries}
\vspace{-7pt}

We denote scalars by lowercase letters, vectors by lowercase boldface letters, matrices by boldface uppercase letters, and sets by uppercase script letters. 

We denote the ground set by $\Pt = \{ 1, \cdots, p\}$, and its power set by  $2^\Pt$.  The $i$-th entry of a vector $\x$ is $x_i$, the projection of $\x$ over a set $\sss \subseteq \Pt$ is $\x_\sss$, i.e., $(\x_\sss)_i=0, \forall i \not \in \sss$. The vector containing the positive part of $\x$ is denoted by $\x_+ = \min\{ \x, 0\}$ (min taken element wise). The absolute value of $|\x|$ is taken element wise. Similarly, the comparison $\x \geq \y$ is taken element wise, i.e., $x_i \geq y_i, \forall i\in \Pt$. For $q \geq 1$, the $\ell_q$-norm of a vector $\x \in \R^p$ is given by $\| \x \|_q =( \sum_{i=1}^p x_i^q)^{1/q}$, and $\| \x \|_\infty = \max_{i}\{ |x_i| \}$.

We call the set of non-zero elements of a vector $\x$ the support, denoted by $\supp(\x) = \{ i: x_i \neq 0\}$. For binary vectors $\s \in \{0,1\}^p$, with a slight abuse of notation, we will use $\s$ and $\supp(\s)$ interchangeably; for example, given a set function $F$, we write $F(\s) = F(\supp(\s))$. We let $\mathds{1}_p$ be the vector in $\R^p$ of all ones, and $\I_p$ the $p \times p$ identity matrix. We drop subscripts whenever the dimensions are clear from the context. In particular, $\mathds{1}_{\supp(\x)}=(\1_p)_{\supp(\x)}$ denotes the projection of $\mathds{1}_p$ over the set $\supp(\x)$.

We introduce some definitions that are used in the sequel.
\begin{definition}[Submodularity]\label{def:submodular} 
A set function $F: 2^{\Pt} \rightarrow \R$ is submodular iff it satisfies the following diminishing returns property: $\forall \sss \subseteq \mathcal{T} \subseteq \Pt, \forall e \in \Pt \setminus \mathcal{T} ,$
 $F(\sss \cup \{e\}) - F(\sss) \geq F(\mathcal{T} \cup \{e\}) - F(\mathcal{T})$.
\end{definition}

\begin{definition}[Fenchel conjugate]
Given a function $g: \R^p \rightarrow \R \cup \{ + \infty\}$, its Fenchel conjugate, $g^\ast: \R^p \rightarrow \R \cup \{ + \infty\}$, is defined as:
$$ g^\ast(\y) := \sup_{\x \in \dom(g)} \x^T \y - g(\x) $$ 
where $\dom(g):= \{ \x : g(\x) < + \infty\}$.
The Fenchel conjugate of the Fenchel conjugate of a function $g$ is called the biconjugate, and is denoted by $g^{\ast \ast}$.
\end{definition}

\begin{definition}[Total unimodularity]
A matrix $\M \in \R^{l \times m}$ is totally unimodular (TU) iff the determinant of every square submatrix of $\M$ is $0$ or $\pm 1$.
\end{definition}

In what follows, some proofs have been omitted due to lack of space; see the supplementary material.
\vspace{-7pt}
\section{A generative view of sparsity models}
\vspace{-7pt}
\subsection{Foundations}
\vspace{-3mm}
We can describe the simplicity and the structured constraints in the proto-problem by encoding them concisely into a combinatorial set function $F$ on the support of the unknown parameter \cite{bach2010structured, obozinski2012convex}. Hence, we can reduce our task of finding the tightest surrogate convex function to determining the convex envelope, i.e., the largest convex lower bound of $F(\supp(\x))$, which is given by its biconjugate. 

Let us first identify a sufficient condition for tractable computation of the convex envelope of $F(\supp(\x))$. 

\begin{lemma}\label{prop:Main_Property}
Given a set function $F:  2^{\Pt} \rightarrow \R \cup \{ + \infty\}$, let $g(\x)=F(\supp(\x))$. If \vspace{-3mm}
\begin{enumerate}\denselist
\item[A1.] $F$ admits a proper ($\dom(f) \neq \emptyset$) lower semi-continuous (l.s.c.) convex extension $f$, i.e., $f(\s) = F(\s), \forall \s \in \{ 0,1 \}^p$;
\item[A2.] $\max_{\s \in \{ 0,1\}^p } |\y|^T \s - f(\s) = \max_{\s \in [ 0,1]^p } |\y|^T \s - f(\s)$, $\forall \y \in \R^p$;
\item[A3.] $\min_{\s \in [0,1]^p} \{f(\s): \s \geq |\x|\}$ can be efficiently minimized, $\forall \x \in \R^p$; \vspace{-3mm}
\end{enumerate}
then the biconjugate of $g(\x)$  over the unit $\ell_\infty$-ball can be efficiently computed.
\end{lemma}

It is also interesting to compute the biconjugate of $g(\x)$ over other unit balls in the Euclidean space, which we will not discuss in this paper.  The proof of Lemma \ref{prop:Main_Property} is elementary, and is provided for completeness:

%
%

\begin{proof}{[Lemma \ref{prop:Main_Property}]} It holds that
\begin{align*}
 g^\ast(\y) &= \sup_{\| \x \|_\infty \leq 1} \x^T \y - F(\supp(\x)) \\
  &= \sup_{\s \in \{ 0,1\}^p} \sup_{\scriptsize \colcst{\| \x \|_\infty \leq 1 \\ \mathds{1}_{\supp(\x)} = \s}} \x^T \y - F(\s) \\
  &= \max_{\s \in \{ 0,1\}^p} | \y|^T \s - F(\s) \tag{by H\"{o}lder's inequality}\\
  &= \max_{\s \in [0,1]^p} | \y|^T \s - f(\s) \tag{by A1 and A2}
\end{align*}
The conjugate is a discrete optimization problem which, in general, is hard to solve. Assumption $A2$ guarantees that its convex relaxation has integral optimal solutions, otherwise the last equality will only hold as an upper bound.
\begin{align*}
 g^{\ast \ast}(\x) &= \sup_{\y \in \R^p} \x^T \y - g^\ast(\y) \\
 &= \sup_{\y\in \R^p }  \min_{\s \in [0,1]^p}  \y^T \x -  | \y|^T \s + f(\s)  \\
  &\stackrel{\star}{=}  \min_{\s \in [0,1]^p} \sup_{\scriptsize \colcst{\y\in \R^p \\ \sign(\y) = \sign(\x)}}   |\y|^T (|\x| -  \s) + f(\s)  \\
  &=\small \begin{cases} \min_{\scriptsize  \colcst{\s \in [0,1]^p \\ \s \geq |\x|}}  f(\s)  &\text{if  $\x \in [-1,1]^p \cap \dom(f)$}\\
 \infty, &\text{otherwise.}
  \end{cases}
 \end{align*}
 
 Given assumption $A1$, $(\star)$ holds by Sion's minimax theorem \cite[Corollary 3.3]{sion1958general}. Assumption $A3$ guarantees that the final convex minimization problem is tractable. 
\end{proof}

%
%

\begin{remark}\label{rmk: lower_bound}
 It is worth noting that, without assumption $A2$, the resulting convex function will still be a convex lower bound of $g(\x)$, albeit not necessarily the \emph{tightest} one.
\end{remark}

\begin{remark}
Note that in lemma \ref{prop:Main_Property}, we had to restrict the biconjugate over the box $[-c,c]^p$ (with $c=1$), otherwise it would evaluate to a constant. In the sequel, unless otherwise stated, we assume $c=1$  without loss of generality.
\end{remark}

In general, computing even the conjugate is a hard problem. If the chosen combinatorial penalty has a tractable conjugate, its envelope can be numerically approximated by a subgradient method \cite{jojic2011convex}. 

\subsection{Submodular sparsity models}

In the generative modeling approach, non-decreasing submodular functions provide a flexible framework that is quite popular. While the best known method for checking submodularity has sub-exponential time complexity \cite{seshadhri2014submodularity}, we can often identify submodular structures by inspection, or we can restrict ourselves to known submodular models that closely approximate our objectives. 

In the light of Lemma \ref{prop:Main_Property}, submodular functions (cf., Def.\ \ref{def:submodular}) indeed satisfy the three assumptions, which allow the tractable computation of tight convex relaxations. For instance, the convex envelope of a submodular non-decreasing function is given by its Lov\'asz extension \cite{bach2010structured}, and optimization with the resulting convex regularizer can be done efficiently. In fact, the corresponding proximity operator is equivalent to solving a submodular minimization problem (SFM) using the minimum-norm point algorithm \cite{bach2010structured}, which empirically runs in ${O}(p^2)$-time \cite{bach2010structured}. However, recent results show that in the worst-case analysis, min-norm point algorithm solves SFM in $O(p^7)$-time \cite{Chakrabarty2014}.

\subsection{$\ell_q$-regularized combinatorial sparsity models} \label{sec:lp-reg-models}
In some applications, we may want to control not only the location of the non-zero coefficients, but also their magnitude. In this setting, it makes sense to consider a combination of combinatorial penalties with continuous regularizers. In particular, functions of the form $\mu F(\supp(\x)) + \nu \|\x \|_q$ are studied in \cite{obozinski2012convex}. The positive homogeneous convex envelope of a $\ell_q$-regularized set function is then given by the dual of $\Omega^\ast_q(\y)= \max_{\s \in \{ 0,1\}^p,\s \neq 0} \frac{\| \y_{\supp(\s)}\|_q  }{F(\s)^{1/q}} $, which can be computed efficiently, for example, in the special case where $F$ is submodular. However, if we only seek to enforce the combinatorial structure, using this approach will fail, as discussed in Section \ref{sec:Dispersive_Models}.
 
\section{Totally unimodular sparsity models}
Combinatorial descriptions that satisfy Lemma \ref{prop:Main_Property} are not limited to submodular functions. Indeed, we can intuitively model the classical sparsity penalties that encourage the simplest support subject to  structure constraints via basic linear inequalities \cite{kyrillidis2012combinatorial,nirav2013tractability}. When the matrix encoding the structure is TU, such models admit tractable convex relaxations that are tight, which is supported by the following.

\begin{lemma}[\cite{nemhauser1999integer}] \label{prop:integral_polyhedra}
Given a TU matrix $\M \in \R^{l \times m}$, an integral vector $\cb \in \mathbb{Z}^l$, and a vector $\boldsymbol{\theta} \in \R^m$. The linear program (LP) $\max_{\betabf \in [0,1]^m} \{ \boldsymbol{ \theta}^T \betabf : \M \betabf \leq \cb\}$ has integral optimal solutions.
\end{lemma}

Let us first provide a simple linear template for TU models:

\begin{definition}[TU penalties]\label{def:TU_penalty}
We define TU penalties as discrete penalties over the support of $\x$ that can be written as
\[  \small g_{\TU}(\x):=  \min_{\omegabf \in \{ 0,1\}^M} \{ \db^T\omegabf + \e^T\s  : \M \betabf \leq \cb, \mathds{1}_{\supp(\x)} = \s\}
\]
for all feasible $\x$, and $ g_{\TU}(\x) = \infty$ otherwise, where $\M \in \R^{l \times (M + p)}$ is a TU matrix, $\betabf = \colvec{\omegabf \\ \s} $,  $\omegabf \in \{ 0,1\}^M$ is a vector of binary variables useful for modeling latent variables , $\db \in \R^M$ and $\e \in \R^p$ are arbitrary weight vectors, and $\cb \in \mathbb{Z}^{l}$ is an integral vector.
\end{definition}
 
By Lemma \ref{prop:integral_polyhedra}, it follows that TU penalties satisfy the sufficient conditions described in Lemma \ref{prop:Main_Property}, where the convex extension is the function itself, and the resulting convex envelope is given below. 
\begin{proposition}[Convexification of TU penalties]\label{eq:conv_TU}
The convex envelope of a TU penalty is given by the following LP:
\begin{equation*} 
\small  g_{\TU}^{\ast \ast}(\x)  =  \min_{\s \in [0,1]^p, \omegabf \in [ 0,1]^M} \{ \db^T\omegabf + \e^T\s : \M \betabf \leq \cb,  |\x| \leq \s\} 
\end{equation*} 
for all feasible $\x$, and $g_{\TU}^{\ast \ast}(\x)=\infty$ otherwise. 
\end{proposition}

Note that when the matrix $M$ in Definition \ref{def:TU_penalty} is not TU, the above LP is still useful, since it is a convex lower bound of the penalty, despite being non-tight as noted in Remark \ref{rmk: lower_bound}. 

\begin{remark}
 The simplicity description does not need to be a linear function of $\omegabf$ and $\s$. We can often find TU descriptions of higher order interactions that can be ``lifted'' to result in the linear TU penalty framework (cf., Section \ref{sec:dispersive_pairwise}).
\end{remark}

\begin{remark}
A weaker sufficient condition for lemma \ref{prop:integral_polyhedra} to hold is for the system $\M \betabf \leq \cb$ to be total dual integral \cite{Giles1979} (e.g., submodular polyhedra). Then, penalties of the form described in Definition \ref{def:TU_penalty} will again satisfy Lemma \ref{prop:Main_Property}.
\end{remark}

Besides allowing tractable tight convexifications, the choice of TU penalties is motivated by their ability to capture several important structures encountered in practice. In what follows, we study several TU penalties and their convex relaxations. We present a reinterpretation of several well-known convex norms in the literature, as well as introduce new ones.

\section{Group sparsity}

Group sparsity is an important class of structured sparsity models that arise naturally in machine learning applications (cf., \cite{wainwright2014structured} and the citations therein), where prior information on $\x^\natural$ dictates certain groups of variables to be selected or discarded together.

A group sparsity model thus consists of a collection of potentially overlapping groups  $\GG= \{ \G_1, \cdots, \G_M\}$ that cover the ground set $\Pt$, where each group $\G_i \subseteq \Pt$ is a subset of variables. A group structure construction immediately supports two compact graph representations (c.f., Figure \ref{fig:graph_rep}).

First, we can represent $\GG$ as a bipartite graph \cite{baldassarre2013group}, where the groups form one set, and the variables form the other. A variable $i \in \Pt$ is connected by an edge to a group $\G_j \in \GG$ iff $i \in \G_j$. We denote by $\Bbd \in \{0,1\}^{p\times M}$ the biadjacency matrix of this bipartite graph; $B_{ij} =1$ iff $i \in \G_j$, and by $\E \in \{0,1\}^{|\mathcal{E}| \times (M+p)}$ its edge-node incidence matrix; $E_{ij} = 1$ iff the vertex $j$ is incident to the edge $e_i \in \mathcal{E}$.
Second, we can represent $\GG$ as an intersection graph \cite{baldassarre2013group}, where the vertices are the groups $\G_i \in \GG$. Two groups $\G_i$ and $\G_j$ are connected by an edge iff $\G_i \cap \G_j \neq \emptyset$. This structure makes it explicit whether groups themselves have cyclic interactions via variables, and identifies computational difficulties. 

\tikzstyle{vnode}=[circle,draw=black,fill=white,thick, minimum size=6pt, inner sep=0pt]
\tikzstyle{vnodeholder}=[circle,draw=black,fill=white,thick, minimum size=1pt, inner sep=0pt]
\tikzstyle{gnode}=[rectangle,draw=black,fill=white,thick, minimum size=6pt, inner sep=0pt]
\tikzstyle{gnodeholder}=[rectangle,draw=black,fill=white,thick, minimum size=1pt, inner sep=0pt]
\begin{figure}
\begin{tabular}{cc}
\hspace{-20pt}

\begin{tikzpicture}[scale=.5, transform shape,-,>=stealth',shorten >=1pt,auto,node distance=1.5cm, semithick]

\node[vnode] (v1) at (0,0) [label=above:$1$] [label=left:variables] {};
\node[vnode] (v2) at (1.1,0) [label=above:$2$] {};
\node[vnode] (v3) at (2.2,0) [label=above:$3$] {};
\node[vnode] (v4) at (3.3,0) [label=above:$4$] {};
\node[vnode] (v5) at (4.4,0) [label=above:$5$] {};
\node[vnode] (v6) at (5.5,0) [label=above:$6$] {};
\node[vnode] (v7) at (6.6,0) [label=above:$7$] {};

\node[gnode] (g1) at (0,-1.8) [label=below:$\G_1$] [label=left:groups] {};
\node[gnode] (g2) at (1.5,-1.8) [label=below:$\G_2$] {};
\node[gnode] (g3) at (3,-1.8) [label=below:$\G_3$] {};
\node[gnode] (g4) at (4.5,-1.8) [label=below:$\G_4$] {};
\node[gnode] (g5) at (6,-1.8) [label=below:$\G_5$] {};
	
\draw (v2) to (g1);\draw (g1) to (v2);

\draw (v1) to (g2);\draw (g2) to (v1);
\draw (v3) to (g2);\draw (g2) to (v3);
\draw (v4) to (g2);\draw (g2) to (v4);

\draw (v2) to (g3);\draw (g3) to (v2);
\draw (v3) to (g3);\draw (g3) to (v3);
\draw (v5) to (g5);\draw (g5) to (v5);

\draw (v5) to (g4);\draw (g4) to (v5);
\draw (v6) to (g4);\draw (g4) to (v6);

\draw (v6) to (g3);\draw (g3) to (v6);
\draw (v7) to (g5);\draw (g5) to (v7);
%

\end{tikzpicture}
&
\begin{tikzpicture}[scale=.5, transform shape,-,>=stealth',shorten >=1pt,auto,node distance=2.8cm, semithick]
  \tikzstyle{every state}=[fill=blue!20,draw=none,text=black]
	
	\node[gnode] (g1) at (0,1.5) [label=above:$\G_1$] {};
	\node[gnode] (g2) at (0,0) [label=below:$\G_2$] {};
	\node[gnode] (g3) at (2.25,0.75) [label=below:$\G_3$] {};	
	\node[gnode] (g4) at (4.5,0.75) [label=below:$\G_4$] {};	
	\node[gnode] (g5) at (6.75,0.75) [label=below:$\G_5$] {};	
	  
	\path (g1) edge node {$\{2\}$} (g3);
	\path (g2) edge node {} (g3);
	\path (g3) edge node {$\{6\}$} (g4);
	\path (g4) edge node {$\{5\}$} (g5);
	
	\path (g3) edge node {} (g1);
	\path (g3) edge node {$\{3\}$} (g2);
	\path (g4) edge node {} (g3);
	\path (g5) edge node {} (g4);
	
\end{tikzpicture}
\end{tabular}
\caption{\small (Left) Bipartite graph representation, (Right) Intersection graph representation of the group structure $\small \GG=\{ \G_1 = \{2\}, \G_2 = \{1,3,4\}, \G_3 = \{2,3,6\}, \G_4 = \{5,6\}, \G_5 = \{5,7\}\}$ } \label{fig:graph_rep}
\end{figure}
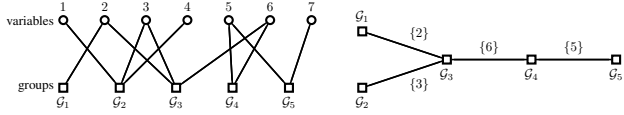

\subsection{Group intersection sparsity} \label{sec:group_count}
In group sparse models, we typically seek to express the support of $\x^\natural$ using only few groups. One natural penalty to consider then is the non-decreasing submodular function that sums up the weight of the groups intersecting with the support $F_{\cap}(\sss) = \sum_{\G_i \in \GG, \sss \cap \G_i \neq \emptyset} d_i$. The convexification of this function results in the $\ell_\infty$-group lasso norm (also known as $\infty$-CAP penalties) \cite{jenatton2011structured, zhao2006grouped}, as shown in \cite{bach2010structured}. 

We now show how to express this penalty as a TU penalty.

\begin{definition}[Group intersection sparsity] \label{def:group_intersection}
\[  g_{\GG, \cap}(\x) :=  \min_{\omegabf \in \{ 0,1\}^M} \{ \db^T \omegabf : \Hb \betabf \leq 0, \mathds{1}_{\supp(x)} =\s \} \]
where $\Hb$ is the following matrix: 
\[ \Hb := \colvec{  -\I_M, \Hb_1  \\  -\I_M, \Hb_2 \\ \cdots \\  -\I_M, \Hb_p}, ~\Hb_k(i,j) = \begin{cases} 1 & \text{if $j=k, j \in \G_i$} \\
0 & \text{otherwise}
\end{cases}\]
and the vector $\db \in \R_+^M$ here corresponds to positive group weights. Recall that $\betabf = \colvec{\omegabf \\ \s}$, and thus $\Hb\betabf \leq 0$ simply corresponds to $s_j \leq w_i, \forall j \in \G_i$.
\end{definition}
$g_{\GG, \cap}(\x)$ indeed sums up the weight of the groups intersecting with the support, since for any coefficient in the support of $\x$ the constraint $\Hb \betabf \leq 0$ forces all the groups that contain this coefficient to be selected.

Here, $\Hb$ is TU, since each row of $\Hb$ contains at most two non-zero entries, and the entries in each row with two non-zeros sum up to zero, which is a sufficient condition for total unimodularity \cite[Proposition 2.6]{nemhauser1999integer}.

\begin{proposition}[Convexification]\label{prop:conv_grpInter}
The convex envelope of $g_{\GG, \cap}(\x)$ over the unit $\ell_\infty$-ball is
$$g_{\GG, \cap}^{\ast \ast}(\x) = \begin{cases}  \sum_{\G_i \in \GG} d_i \| \x_{\G_i} \|_\infty  &\text{if $\x \in [-1,1]^p$}\\
\infty &\text{otherwise}
\end{cases}$$
\end{proposition}

\begin{figure}
\centering
\includegraphics[clip=true,trim=0cm 8cm 0cm 4cm,scale=.18]{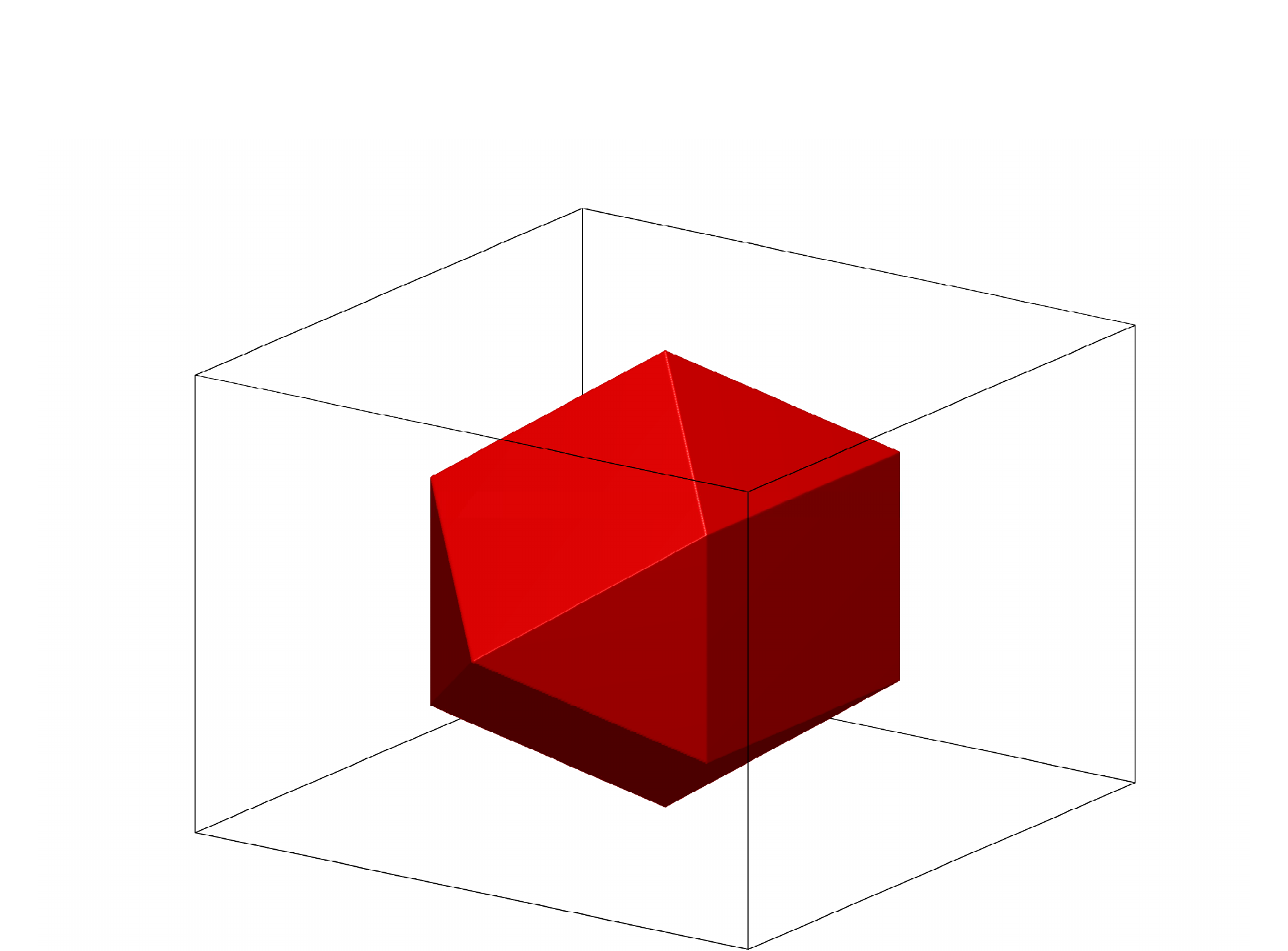}
\caption{Unit norm ball of $g_{\GG, \cap}^{\ast \ast}$, for $\GG=\{ \{1,2\}, \{ 2,3\}\}$, unit group weights $\db=\1$.}
\end{figure}

\subsection{Minimal group cover}\label{sec:grp_cover}
The groups intersections penalty induces supports corresponding to the intersection of the complements of groups, while in several applications, it is desirable to explain the support of $\x^\natural$ as the union of groups in $\GG$. In particular, we can seek the minimal set cover of $\x^\natural$:

\begin{definition}[Group $\ell_0$-``norm'', \cite{baldassarre2013group}]\label{def:group_cover}
The group  $\ell_0$-``norm'' computes the weight of the minimal weighted set cover of $\x$ with group weights $\db \in \R_+^M$:
\vspace{-5pt}
\[ g_{\GG,0}(\x) := \min_{\omegabf \in \{0,1\}^M} \{ \db^T \omegabf: \Bbd \omegabf \geq  \mathds{1}_{\supp(\x)}\}, \]
where  $\Bbd$ is the biadjacency matrix of the bipartite graph representation of $\GG$.
\end{definition}
Note that computing the group  $\ell_0$-``norm'' is NP-Hard, since it corresponds to the minimum weight set cover problem. $g_{\GG,0}(\x)$ is a penalty that was previously considered in \cite{baldassarre2013group, obozinski2011group, huang2011learning}, and the latent group lasso was proposed in \cite{obozinski2011group} as a potential convex surrogate for it, but it was not established as the tightest possible convexification. 

The group $\ell_0$-``norm'' is \emph{not} a submodular function, but if $\Bbd$ is TU, it is a TU penalty, and thus is admits a tight convex relaxation. We show below that the convex envelope of the group  $\ell_0$-``norm'' is indeed the $\ell_\infty$-latent group norm. It is worth noting that the $\ell_q$-latent group lasso was also shown in \cite{obozinski2012convex} to be the positive homogeneous convex envelope of the $\ell_q$-regularized group  $\ell_0$-``norm'', i.e., of  $\mu g(\x)_{\GG,0} + \nu \|\x \|_q$.
\begin{proposition}[Convexification] \label{prop:conv_grpCover}
When the group structure leads to a TU biadjacency matrix $ \Bbd$, the convex envelope of the group $\ell_0$-``norm'' over the unit $\ell_\infty$-ball is
\[\small g_{\GG,0}^{\ast \ast}(\x) = \begin{cases} \min_{\omegabf \in [0,1]^M} \{ \db^T \omegabf :  \Bbd \omegabf \geq |\x| \} &\text{if  $\x \in [-1,1]^p$} \\
\infty &\text{ otherwise}
\end{cases}\]
\end{proposition}

\begin{figure}
\centering
\includegraphics[clip=true,trim=0cm 8cm 0cm 4cm,scale=.18]{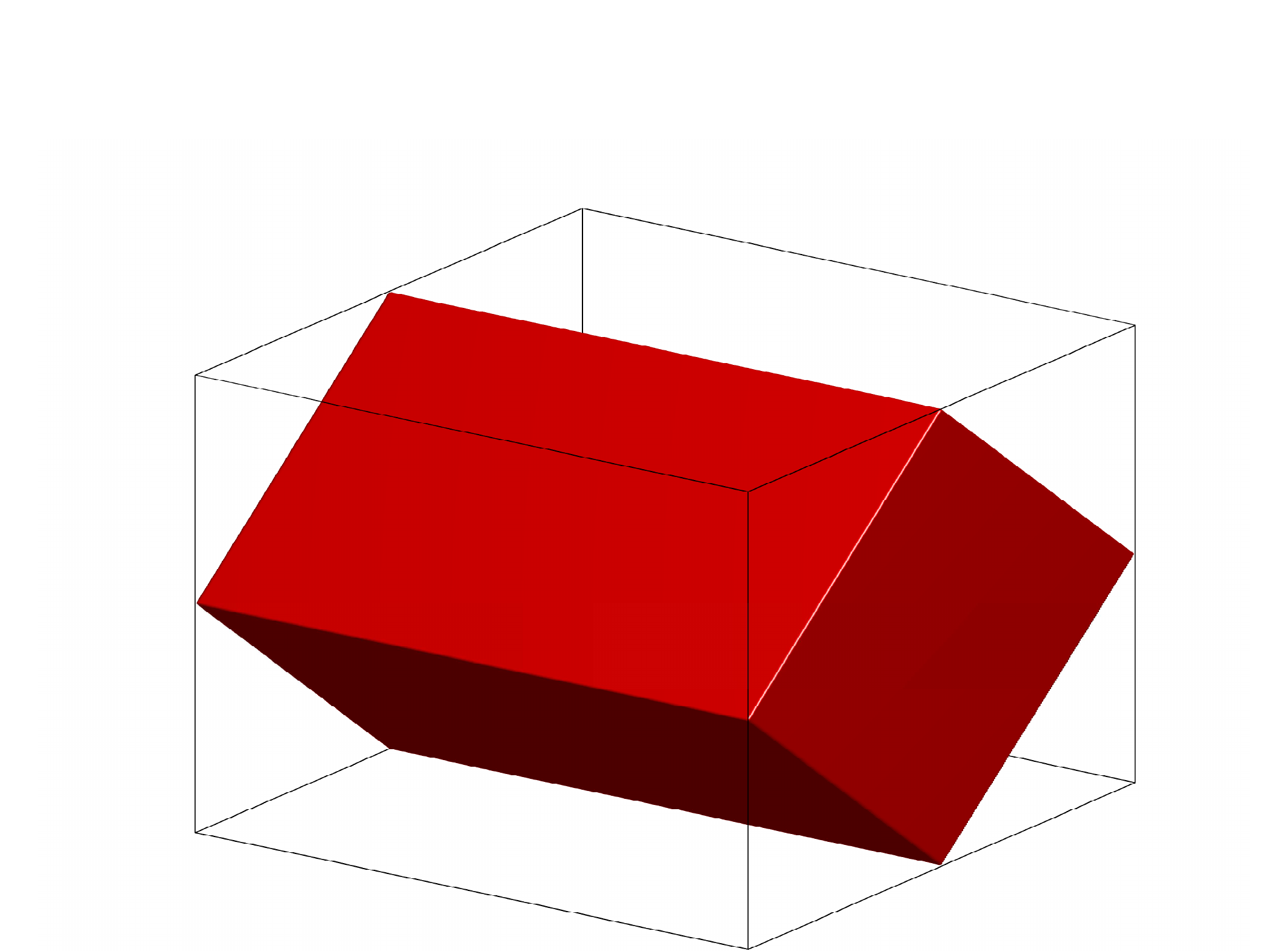}
\caption{Unit norm ball of $g_{\GG, 0}^{\ast \ast}$, for $\GG=\{ \{1,2\}, \{ 2,3\}\}$, unit group weights $\db=\1$.}
\end{figure}

Thus, given a group structure $\GG$, one can check in polynomial time if it is TU \cite{truemper1982alpha} to guarantee that the $\ell_ \infty$-latent group lasso will be the tightest relaxation. 

\begin{remark}\label{rmk: Loopless_grps}
One important class of group structures that leads to a TU matrix $\Bbd$ is given by acyclic groups, as shown in \cite[Lemma 2]{baldassarre2013group}. The induced intersection graph for such groups is acyclic, as illustrated in Figure \ref{fig:graph_rep}. In this case, the $\ell_\infty$-latent group norm is a tight relaxation.
\end{remark}

\subsection{Sparsity within groups}

Both group model penalties we considered so far only induce sparsity on the group level; if a group is selected, all variables within the group are encouraged to be non-zero. In some applications, it is desirable to also enforce sparsity within groups. We thus consider a natural extension of the above two penalties, where each group is weighed by the $\ell_0$-norm of $\x$ restricted to the group.

\begin{definition}[Group models with sparsity within groups]
\[\small g_{\GG, s}(\x) =  \min_{\omegabf \in \{0,1\}^M} \{ \sum_{i=1}^M \omega_i \| \x_{\G_i}\|_0: \M \betabf \leq 0,   \mathds{1}_{\supp(\x)}=\s\} \]
where $\M$ here is either $\M = \Hb$ in Definition \ref{def:group_intersection} or $\M = [-\Bbd, \I_p]$ in Definition \ref{def:group_cover}.
\end{definition}

Unfortunately, this penalty leads to a non-TU penalty, and thus its corresponding convex surrogate given by Proposition \ref{eq:conv_TU} is not guaranteed to be tight.

\begin{proposition}\label{prop:SGL_notTU}
Given any group structure $\GG$, $g_{\GG, s}(\x)$ is not a TU penalty.
\end{proposition}

\begin{proposition}[Convexification] \label{prop:conv_grpLasso_sparse}
The convex surrogate via Proposition \ref{eq:conv_TU} for $g_{\GG, s}(\x)$ with $\M= \Hb$ (i.e., the group intersection model with sparse groups) is given by 
$$\Omega_{\GG, s}(\x) := \sum_{(i,j) \in \Et} (\| \x_{\G_i} \|_\infty + |x_j| -1)_+ $$
for $\x \in [-1,1]^p$, and $\Omega_{\GG, s}(\x) := \infty$ otherwise. Note that $\Omega_{\GG, s}(\x)\le g_{\GG, s}^{**}(\x)$.
\end{proposition}

By construction, the convex penalty proposed by Proposition \ref{prop:conv_grpLasso_sparse} is different from the sparse group lasso in \cite{simon2013sparse}. 

Analogous to the latent group norm, we can seek to convexify the sparsest set cover with sparsity within groups:

\begin{proposition}[Convexification] \label{prop:conv_grpLasso_sparse2}
The convex surrogate via Proposition \ref{eq:conv_TU} for $g_{\GG, s}(\x)$ with $\M = [-\Bbd, \I_p]$ (i.e., the group $\ell_0$-``norm'' with sparse groups) is given by
$$\Omega_{\GG, s}(\x) :=  \min_{\omegabf \in [0,1]^M} \{\sum_{(i,j) \in \Et} (\omega_i + |x_j| -1)_+ : \Bbd \omegabf \geq |\x| \}$$
for $\x \in [-1,1]^p$, and $g(\x)_{\GG, s}=\infty$ otherwise.
\end{proposition}

\subsection{Sparse $G$-group cover}\label{sec:Sarse_grpCover}

In this section, we provide a more direct formulation to enforce sparsity both on the coefficients and the group level. If the true signal $\x^\natural$ we are seeking is a sparse signal covered by at most $G$ groups, it would make sense to look for the sparsest signal with a $G$-group cover that explains the data in \eqref{eq: proto}. This motivates the following natural penalty.

\begin{definition}[Sparse $G$-group cover]\label{def:sparse_group_cover}

\[ \small g_{\GG,G}(\x) := \min_{\omegabf \in \{0,1\}^p} \{ \1^T \s: \Bbd \omegabf \geq \s, \1^T \omegabf \leq G, \s= \mathds{1}_{\supp(\x)}\} \]
where  $\Bbd$ is the biadjacency matrix of the bipartite graph representation of $\GG$.
\end{definition}

If the actual number of active groups is not known, $G$ would be a parameter to tune. Note that $g_{\GG,G}$ is an extension of the minimal group cover penalty (c.f., Section \ref{sec:grp_cover}), where instead of looking for the signal with the smallest cover, we seek the sparsest signal that admit a cover with fewer than $G$ groups. $g_{\GG,G}$ is a TU penalty whenever $\widetilde{\Bbd}=\colvec{\Bbd \\ \1}$is TU \cite[Proposition 2.1]{nemhauser1999integer}, which is the case, for example, when $\Bbd$ is an interval matrix.
 
\begin{proposition}[Convexification] \label{prop:convexification}
When the group structure leads to a TU constraint matrix $\widetilde{\Bbd}$, the convex envelope of $g_{\GG,G}$ over the unit $\ell_\infty$-ball is
\[ g_{\GG,G}^{\ast \ast}(\x) =  \min_{\omegabf \in [0,1]^M} \{ \| \x \|_1 : \Bbd \omegabf  \geq |\x|,  \1^T \omegabf \leq G\}\]
for $\x \in [-1,1]^p$, and  $g_{\GG,G}^{\ast \ast}(\x) =\infty$ otherwise.
\end{proposition}
The resulting convex program thus combines the latent group lasso (c.f., Section \ref{sec:grp_cover}) with the $\ell_1$ norm and provides an alternative to the sparse group lasso in \cite{simon2013sparse}, for the overlapping groups case. In the supplementary material we provide a numerical illustration of its performance.

\subsection{Hierarchical model} \label{sec:Tree}
We study the hierarchical sparsity model, where the coefficients of $\x^\natural$ are organized over a tree $\T$, and the non-zero coefficients form a rooted connected subtree of $\T$ (cf., Figure \ref{fig:hierarchical}). This model is popular in image processing due to the natural structure of wavelet coefficients \cite{jenatton2011proximal, duarte2008wavelet, soni2011efficient}. We can  describe such a hierarchical model as a TU model:

\tikzstyle{notsel}=[circle,draw=black,fill=white,thick, minimum size=6pt, inner sep=0pt]
\tikzstyle{sel}=[circle,draw=black,fill=black,thick, minimum size=6pt, inner sep=0pt]

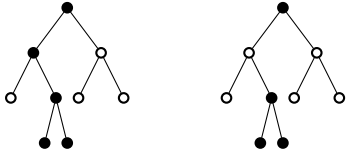
\begin{figure}
\centering
\begin{tabular}{cc} 
\begin{tikzpicture}[level distance=10mm, scale=.6, transform shape]
  \tikzstyle{level 1}=[sibling distance=15mm]
  \tikzstyle{level 2}=[sibling distance=10mm]
  \tikzstyle{level 3}=[sibling distance=5mm]
  \node[sel] {}
    child {node[sel] {} 
  		child{node[notsel] {}} 
		child{node[sel] {} 
			child{node[sel] {}} 
			child{node[sel] {}}}}
    child {node[notsel] {} 
		child{node[notsel] {}} 
		child{node[notsel] {}}};
\end{tikzpicture}
& \hspace{20pt} 
\begin{tikzpicture}[level distance=10mm, scale=.6, transform shape]
  \tikzstyle{level 1}=[sibling distance=15mm]
  \tikzstyle{level 2}=[sibling distance=10mm]
  \tikzstyle{level 3}=[sibling distance=5mm]
  \node[sel] {}
    child {node[notsel] {} 
  		child{node[notsel] {}} 
		child{node[sel] {} 
			child{node[sel] {}} 
			child{node[sel] {}}}}
    child {node[notsel] {} 
		child{node[notsel] {}} 
		child{node[notsel] {}}};
\end{tikzpicture}\\
\end{tabular}
\caption{\label{fig:hierarchical} \footnotesize Valid selection (left), {\em Invalid} selection (right) }
\end{figure}

\begin{definition}[Tree $\ell_0$-``norm'']
We define the penalty encoding the hierarchical model on $\x$ as
\[ g_{T,0}(\x) := \begin{cases}
\| \x \|_0 & \text{if $\Tb  \mathds{1}_{\supp(\x)} \geq 0$} \\
\infty & \text{otherwise}
\end{cases}\]
where $\Tb$ is the edge-node incidence matrix of the directed tree $\T$, i.e., $T_{li}=1$ and $T_{lj}=-1$ iff $e_l=(i,j)$ is an edge in $\T$. It encodes the constraint $s_{\text{parent}} \geq s_{\text{child}}$ for $\s = \mathds{1}_{\supp(\x)}$ over the tree. \end{definition}
This is indeed a TU model since each row of $\T$ contains at most two non-zero entries that sum up to zero \cite[Proposition 2.6]{nemhauser1999integer}.

\begin{proposition}(Convexification) \label{prop:conv_tree}
The convexification of the tree $\ell_0$-``norm'' over the unit $\ell_\infty$-ball is given by $$g^{\ast \ast}_{T,0}(\x)= \begin{cases}
\sum_{\G \in \GG_H} \| x_\G \|_\infty &\text{if $\x \in [-1,1]^p$} \\
 \infty &\text{otherwise}
\end{cases}$$
where the groups $\G \in \GG_{H}$ are defined as each node and all its descendants.
\end{proposition}

\begin{figure}
\centering
\begin{tabular}{cc}
  \includegraphics[height= .1\textheight]{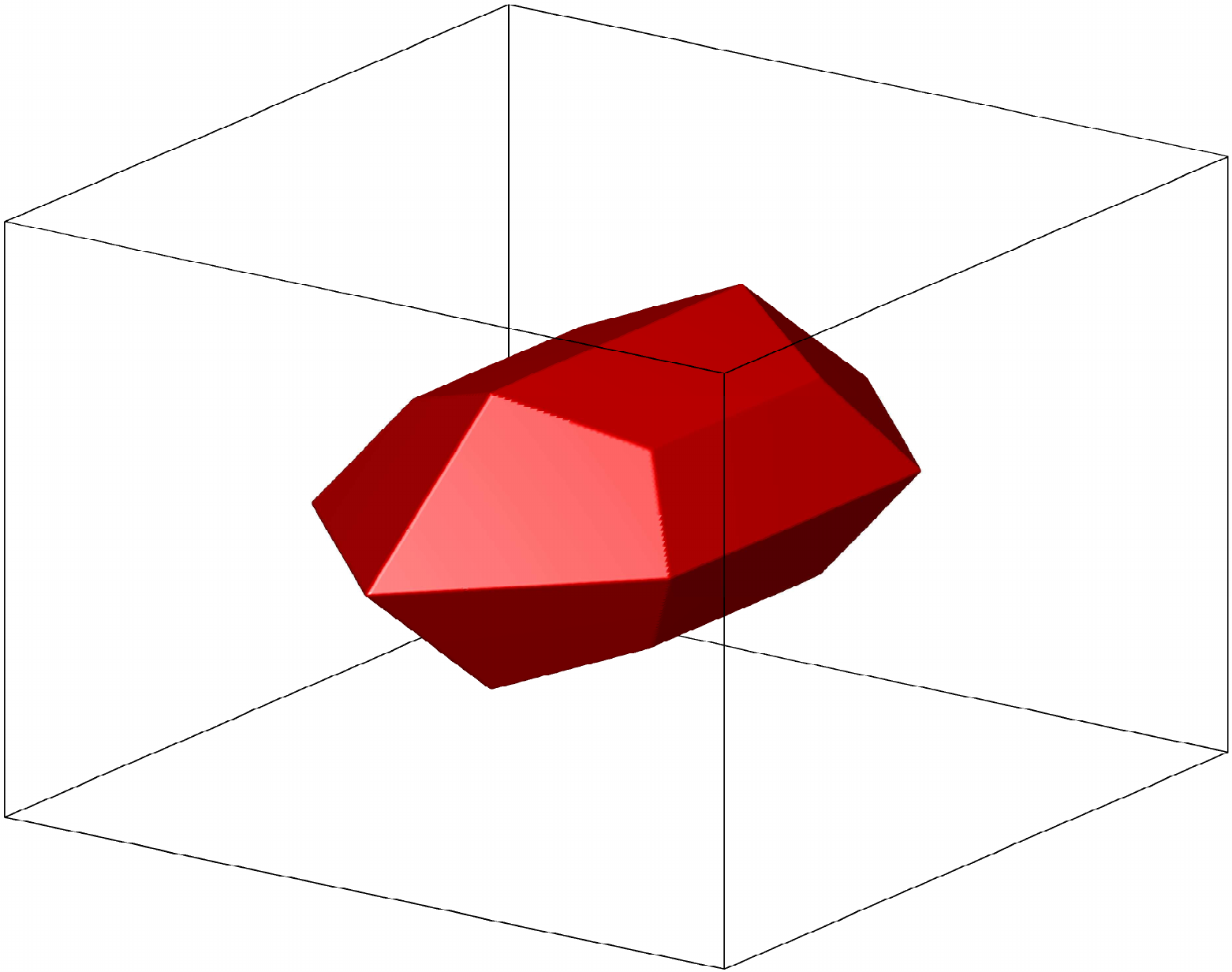} & \includegraphics[height= .1\textheight]{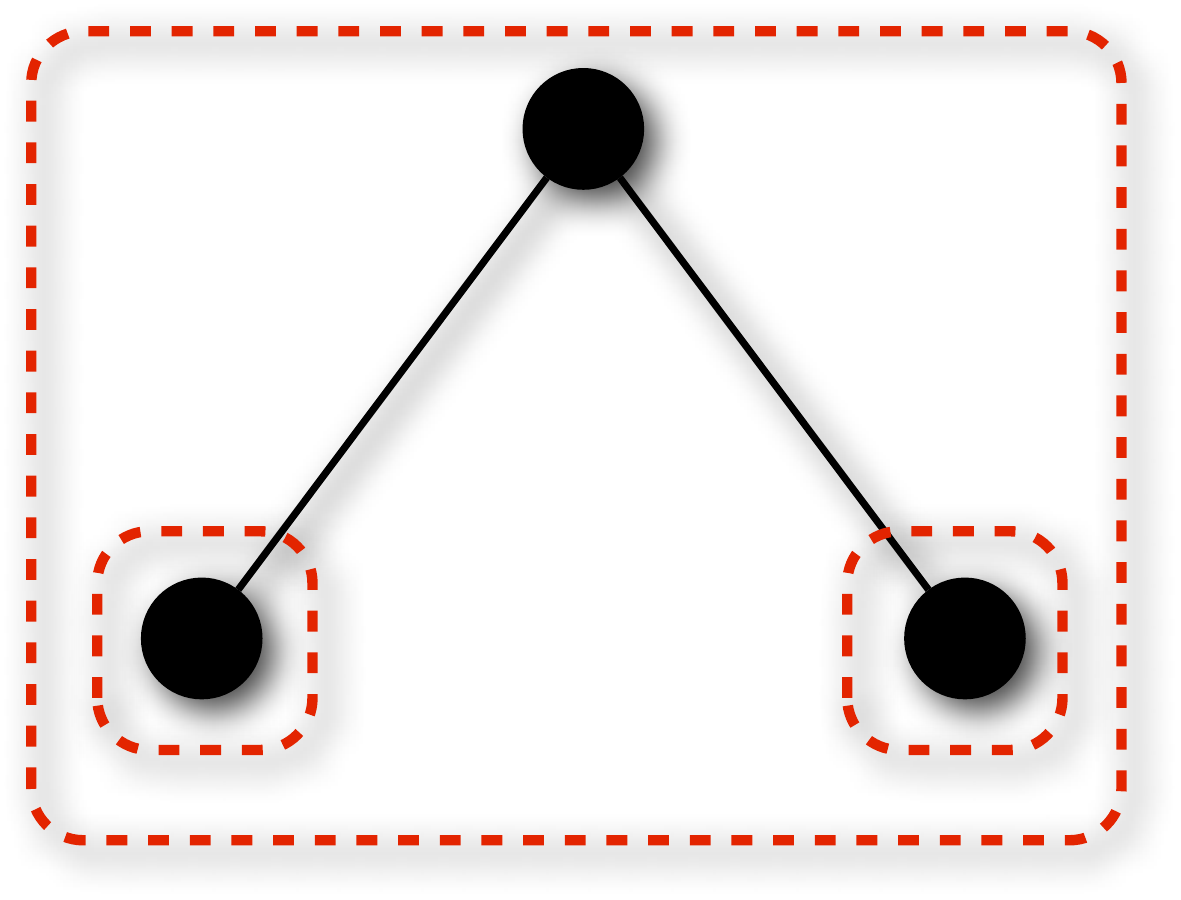}
\end{tabular}
\caption{\footnotesize Unit norm ball of $g^{\ast \ast}_{T,0}$,  $\GG_H=\{ \{1,2,3\}, \{ 2\},\{ 3\}\}$}
\end{figure}

Note that the resulting convex norm is the $\ell_\infty$-hierarchical group norm \cite{jenatton2011proximal}, which is a special case of $\ell_\infty$-group norm we studied in Section \ref{sec:group_count} as the convex envelope of $g_{\GG, \cap}(\x)$. In this sense, $g_{\GG, \cap}(\x)$ is equivalent to $g_{T,0}(\x)$, for the group structure $\GG_H$ (for unit weights).

%

\section{Dispersive sparsity models} \label{sec:Dispersive_Models}
The sparsity models we considered thus far encourage clustering. The implicit structure in these models is that coefficients within a group exhibit a positive, reinforcing correlation. Loosely speaking, if a coefficient within a group is important, so are the others. However, in many applications, the opposite behavior may be true. That is, sparse coefficients within a group compete against each other \cite{zhou2010exclusive, hegde2009compressive, gerstner2002spiking}. 

Hence, we describe models that encourage the dispersion of sparse coefficients. Here, dispersive models still inherit a known group structure $\GG$, which underlie their interactions in the opposite manner to the group models in Section 5.

\subsection{Group knapsack model}

One natural model for dispersiveness allows only a certain budget of coefficients, e.g., only one, to be selected in each group:
\vspace{-10pt} $$ \small F_{\Db}(\sss) = \begin{cases} 0 &\text{if $\sss = \emptyset$} \\1 &\text{if $\max_{\G \in \GG} |\sss \cap \G| \leq 1$} \\ \infty &\text{otherwise} \end{cases}$$
Whenever the group structure forms a \emph{partition} of $\Pt$, \cite{obozinski2012convex} shows that the positive homogeneous convex envelope of the $\ell_q$-regularized group knapsack model, i.e., of $\mu F_{\Db}(\supp(\x)) + \nu \| \x \|_q$, is the exclusive norm in \cite{zhou2010exclusive}. 

In what follows, we prove that $F_{\Db}(\sss)$ is a TU penalty whenever the group structure leads to a TU biadjacency matrix $\Bbd$ of the bipartite graph representation, which includes partition structures. We establish that the $\ell_\infty$-exclusive lasso, is actually the tightest convex relaxation of a more relaxed version of $F_{\Db}(\supp(\x))$ 
for any TU group structure, and not necessarily partition groups.

\begin{definition}[Group knapsack penalty] 
Given a group structure $\GG$ that leads to a TU biadjacency matrix $\Bbd$, $F_{\Db}(\supp(\x))$ can be written as the following TU penalty:
$$  g_{\Db}(\x) := \min_{\omega \in \{ 0,1\} } \{ \omega: \Bbd^T \mathds{1}_{\supp(\x)} \leq \omega \mathds{1}\} $$
if $\Bbd^T \mathds{1}_{\supp(\x)} \leq \mathds{1}$, and $g_{\Db}(\x)=\infty$ otherwise. 
\end{definition}
Note that if $\Bbd$ is TU, $\Bbd^T$ is also TU \cite[Proposition 2.1]{nemhauser1999integer}. Groups that form a partition of $\Pt$ are acyclic, thus the corresponding matrix $\Bbd$ is TU trivially  (cf., Remark \ref{rmk: Loopless_grps}).

Another important example of a TU group structure arises from the simple one-dimensional model of the neuronal signal suggested by \cite{hegde2009compressive}. In this model, neuronal signals are seen as a train of spike signals with some refractoriness period $\Delta \geq 0$, where the minimum distance between two non-zeros is $\Delta$. This structure corresponds to an interval matrix $\Bbd^T =  \Db$, which is TU \cite[Corollary 2.10]{nemhauser1999integer}.
 $$ \small \Db=\begin{bmatrix}
1 & 1 & \cdots & 1 & 1 & 0 & 0 & \cdots & 0 \vspace{0.1cm} \\
0 & 1 & 1& \cdots & 1 & 1 & 0 & \cdots & 0 \vspace{0.1cm} \\
& & & & \ddots  & & & \vspace{0.1cm} \\
0& \cdots & 0 & 0 & 1 & 1 & \cdots & 1 & 1
\end{bmatrix}_{(p - \Delta + 1) \times p}$$

\begin{proposition}[Convexification] \label{prop:conv_disp}
The convex envelope of $g_{\Db}(\x) $ over the unit $\ell_\infty$-ball when $\Bbd^T$ is a TU matrix is given by
$$  g_{\Db}^{\ast \ast}(\x) = \begin{cases} \max_{\G \in \GG} \| \x_\G\|_1 &\text{if  $\x \in [-1,1]^p, \Bbd^T |\x|\leq \mathds{1}$ } \\
\infty &\text{otherwise} \end{cases}$$
\end{proposition}
Notice that the convexification of $g_{\Db}$ is not exactly the exclusive lasso; it has an additional budget constraint $\Bbd^T |\x|\leq \mathds{1}$. Thus in this case, regularizing with the $\ell_q$-norm before convexifying lead to the loss of part of the structure. In fact, the exclusive norm is actually the convexification of a more relaxed version of $g_{\Db}$, where the constraint $\omegabf \in \{0,1\}$ is relaxed to $\omegabf \in \mathbb{Z}, \omegabf \geq 0$.

\subsection{Sparse group knapsack model} \label{sec: sparse_dispersive}
In some applications, it may be desirable to seek the sparsest signal satisfying the dispersive structure. This can be achieved by incorporating sparsity into the group knapsack penalty, resulting in the following TU penalty.
\begin{definition}[Dispersive $\ell_0$-``norm''] 
Given a group structure $\GG$ that leads to a TU biadjacency matrix $\Bbd$, we define the penalty encoding the sparse group knapsack model on $\x$ as 
\[ g_{\Db,0}(\x) := \begin{cases}
\| \x \|_0 & \text{if $\Bbd^T   \mathds{1}_{\supp(\x)} \leq  \mathds{1}$} \\
\infty & \text{otherwise}
\end{cases}\]
\end{definition}
We can compute the convex envelope of $g_{\Db,0}(\x)$ in a similar fashion to Proposition \ref{prop:conv_disp}.
\begin{proposition}(Convexification)
The convexification of the dispersive $\ell_0$-``norm'' over the unit $\ell_\infty$-ball is given by $$g^{\ast \ast}_{D,0}(\x)= \begin{cases}
{ \| \x \|_1} &\text{if $\x \in [-1,1]^p, \Bbd^T |\x |\leq \mathds{1} $} \\
 \infty &\text{otherwise}
\end{cases}$$
\end{proposition}
It is worth mentioning that regularizing with the $\ell_q$-norm here loses the underlying dispersive structure. In fact, the positively homogeneous convex envelope of $\mu g_{\Db,0}(\x) + \lambda \| \x \|_q $ is given by the dual (cf., Section \ref{sec:lp-reg-models}) of $$\Omega^\ast_q(\y)= \max_{\s \in \{ 0,1\}^p,\s \neq 0, \Bbd^T \s \leq \mathds{1}} \frac{\| \y_{\supp(\s)}\|_q  }{(\mathds{1}^T\s)^{1/q}} $$
which is simply the $\ell_1$-norm. To see this, note that $\Omega^\ast_a(\y)^q =\| \y \|_\infty$, since $\frac{\sum_{i \in \sss} | y_i|^q}{|\sss|} \leq \frac{|\sss| \| \y \|_\infty^q}{|\sss|}, \forall \sss \subseteq \Pt$ which is achieved with equality by choosing the vector $\s$ having ones where $\y$ is maximal, and zeros elsewhere. Note that this vector satisfies $\Bbd^T \s \leq \mathds{1}$.
As a result, the regularized convexification boils down to the $\ell_1$-norm, since $\Omega_q(\x) = \sup_{\Omega^\ast_q(\y) \leq 1} \x^T\y  = \| \x\|_1$, while the direct convexification is not even a norm (cf., Figure \ref{fig:disp-ball}). 

We illustrate the effect of this loss of structure via a numerical example in Section \ref{sec: simulation}.

\begin{figure}
\centering
\begin{tabular}{ccc}
	\includegraphics[height= 0.08\textheight]{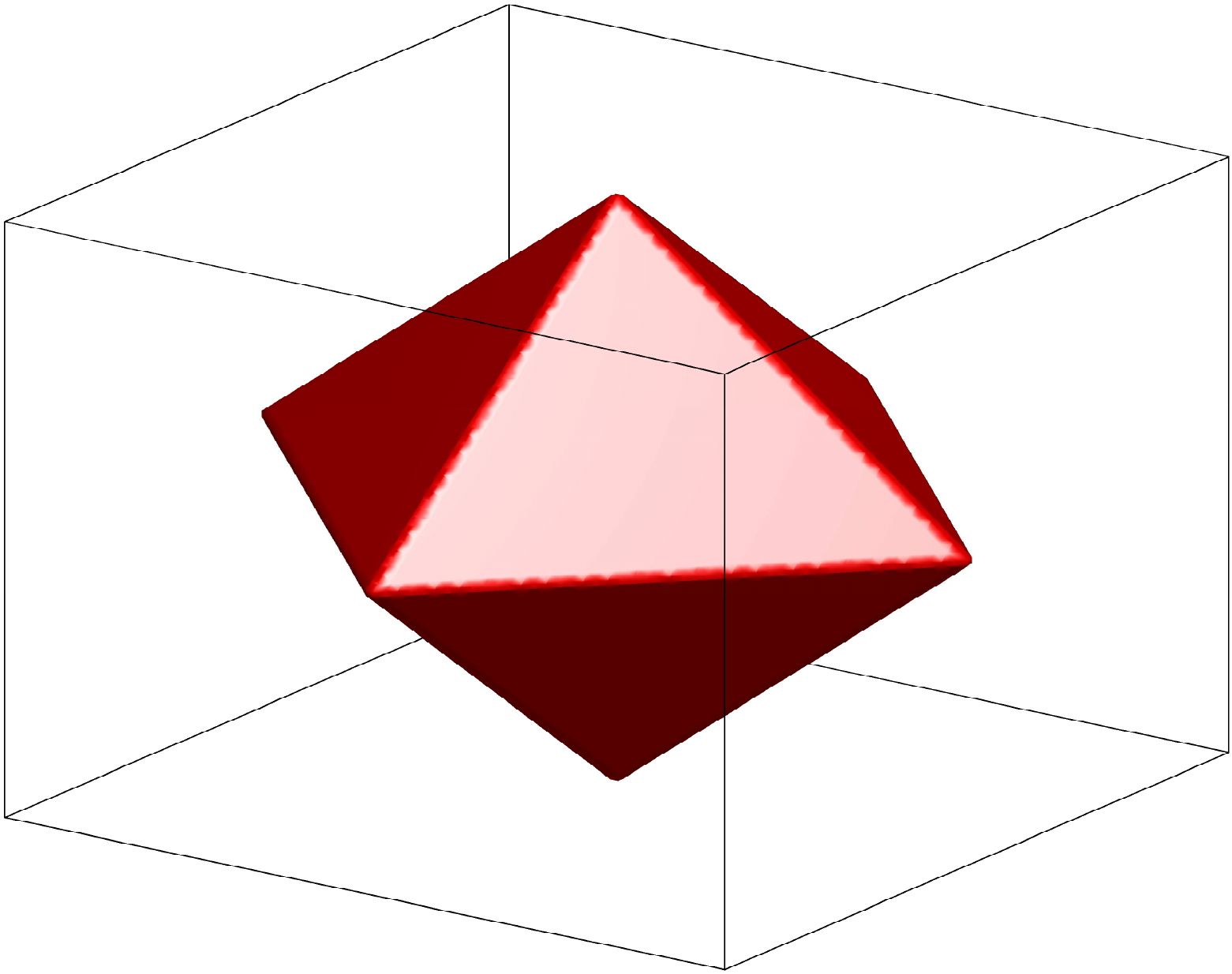} & 
	\includegraphics[height= 0.08\textheight]{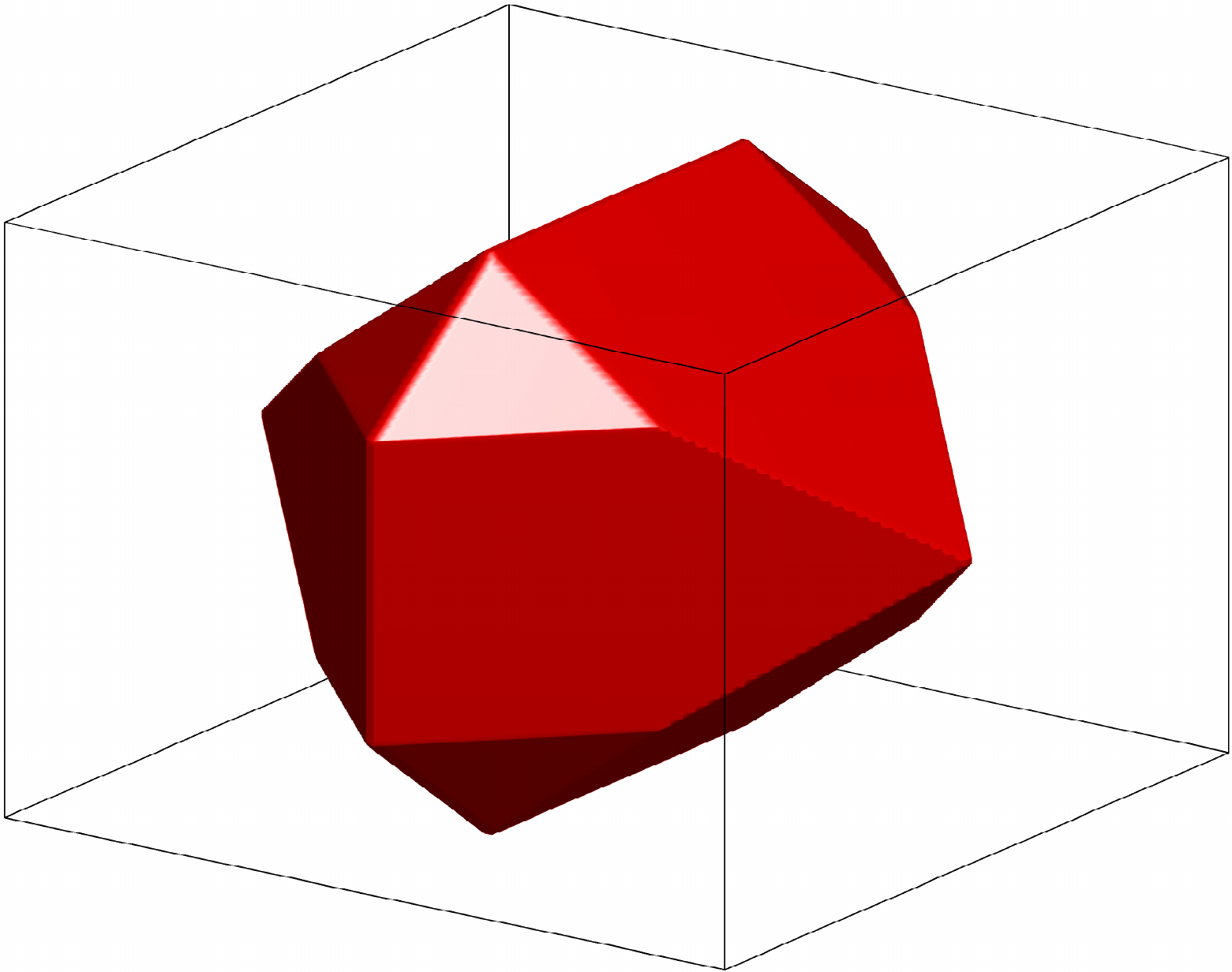} &     
	\includegraphics[height= 0.08\textheight]{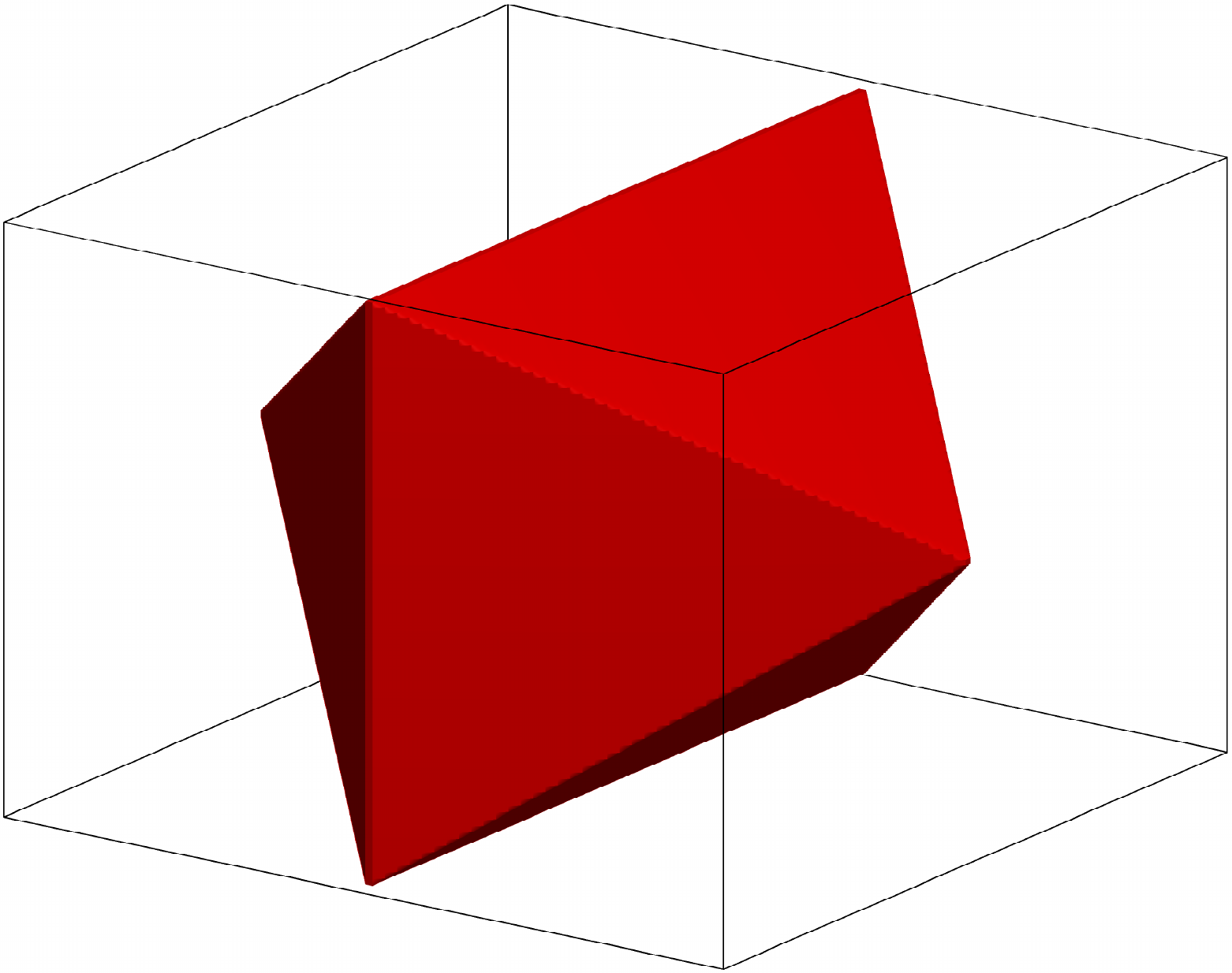} 
\end{tabular}
\caption{ \label{fig:disp-ball} $g^{\ast \ast}_{D,0}(\x)\le 1$ {(left)} $g^{\ast \ast}_{D,0}(\x)\le 1.5$  {(middle)} $g^{\ast \ast}_{D,0}(\x)\le 2$ {(right)} for $\GG=\{ \{1,2\}, \{ 2,3\}\}$}
\end{figure}

\subsection{Graph dispersiveness} \label{sec:dispersive_pairwise}

In this section, we illustrate that our framework is not limited to linear costs, by considering a pairwise dispersive model. We assume that the parameter structure is encoded on a known graph $G(\Pt , \Et )$, where coefficients connected by an edge are discouraged from being on simultaneously.

\begin{definition}[Pairwise dispersive penalty]
Given a graph $G(\Pt , \Et )$ with a TU edge-node incidence matrix $\E_G$ (e.g., bipartite graph), we define the penalty encoding the pairwise dispersive model as
$$ g_{\G, \D}(\x) = \sum_{(i,j) \in \Et} s_i s_j ~\text{  where } ~\s=\mathds{1}_{\supp(\x)}$$
\end{definition}
Note that this function is not submodular; in fact, $g_{\G, \D}(\x)$ is a \emph{supermodular} function.
\begin{proposition}[Convexification] \label{prop:conv_disp_pair}
The convex envelope of $g_{\G, \D}(\x)$ over the unit $\ell_\infty$-ball is
$$ g_{\G, \D}^{\ast \ast}(\x) = \begin{cases} \sum_{(i,j) \in \Et} (|x_i| + |x_j| -1)_+  &\text{if  $\x \in [-1,1]^p$ } \\
\infty &\text{otherwise} \end{cases}$$
\end{proposition}
\begin{proof}
 We use the linearization trick employed in \cite{kaminski2008quadratic} to reduce $g_{\G, \D}(\x)$ to a TU penalty. Let $\s= \mathds{1}_{\supp(\x)}$,
\begin{align*}
g_{\G, \D}(\x) &= \sum_{(i,j) \in \Et} s_i s_j \\
&= \min_{\z \in \{ 0,1\}^{|\Et|}} \{ \sum_{(i,j) \in \Et} z_{ij}: z_{ij} \geq s_i + s_j -1 \} \\ 
&= \min_{\z \in \{ 0,1\}^{|\Et|}} \{ \sum_{(i,j) \in \Et} z_{ij}: \E_G \s \leq \z - \mathds{1} \}
\end{align*} 
Now we can apply Proposition \ref{eq:conv_TU} to compute the convex envelope. The resulting convexification is again \emph{not a norm} (c.f., Figure \ref{fig:graph-disp-ball}).
\end{proof}

\begin{figure}
\centering
\begin{tabular}{cc}

	\includegraphics[height= 0.09\textheight]{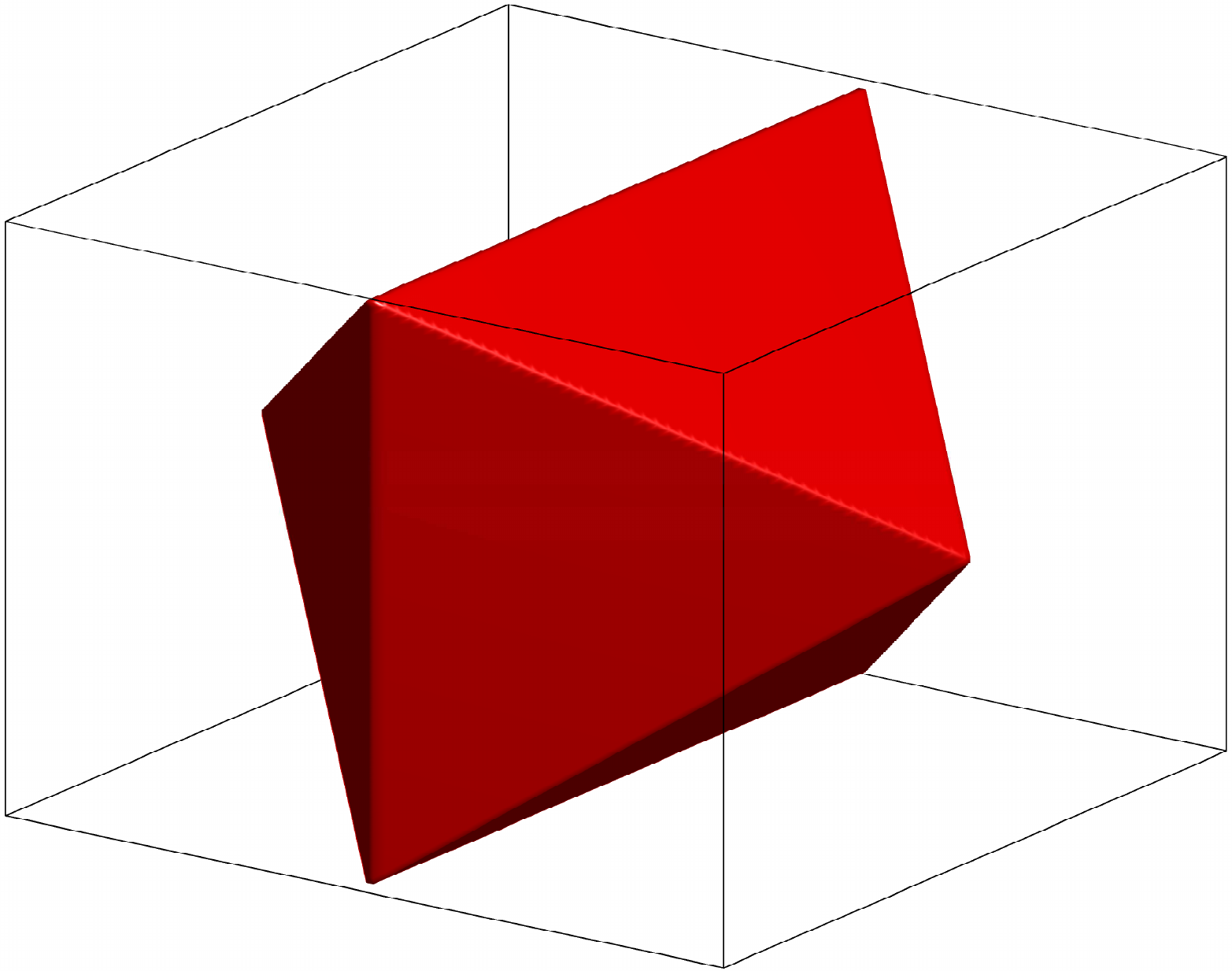}&
	\includegraphics[height= 0.09\textheight]{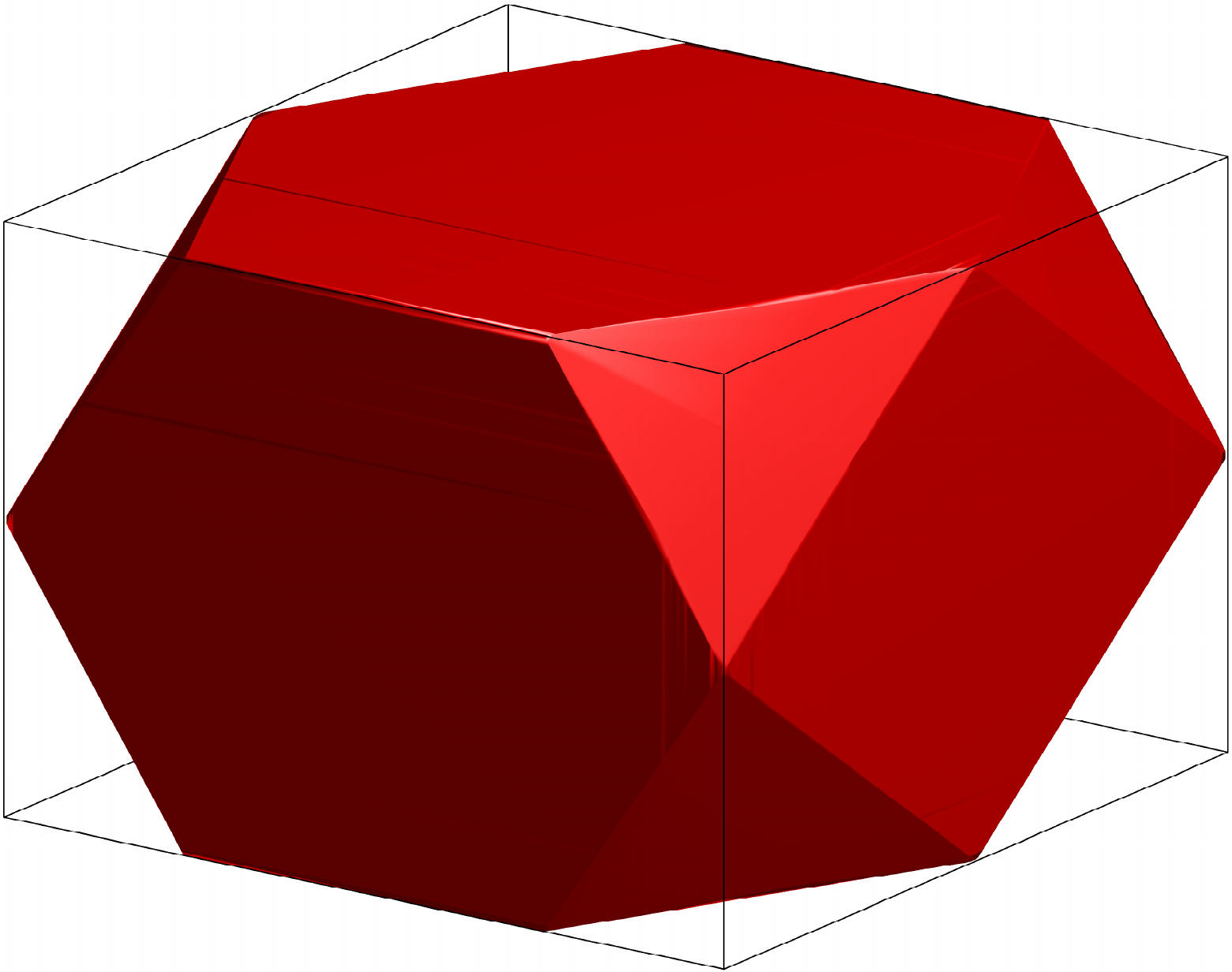}  
 
\end{tabular}
\caption{\label{fig:graph-disp-ball} $g_{\G, \D}^{\ast \ast}(\x) = 0$  {(left)} $g_{\G, \D}^{\ast \ast}(\x) \le 1$ {(right)} for $\Et=\{ \{1,2\}, \{ 2,3\}\}$ (chain graph)}
\end{figure}


\vspace{-10pt}
\section{Numerical illustration} \label{sec: simulation}
\begin{figure} 
\centering
\vspace{-70pt}
\includegraphics[scale=.34]{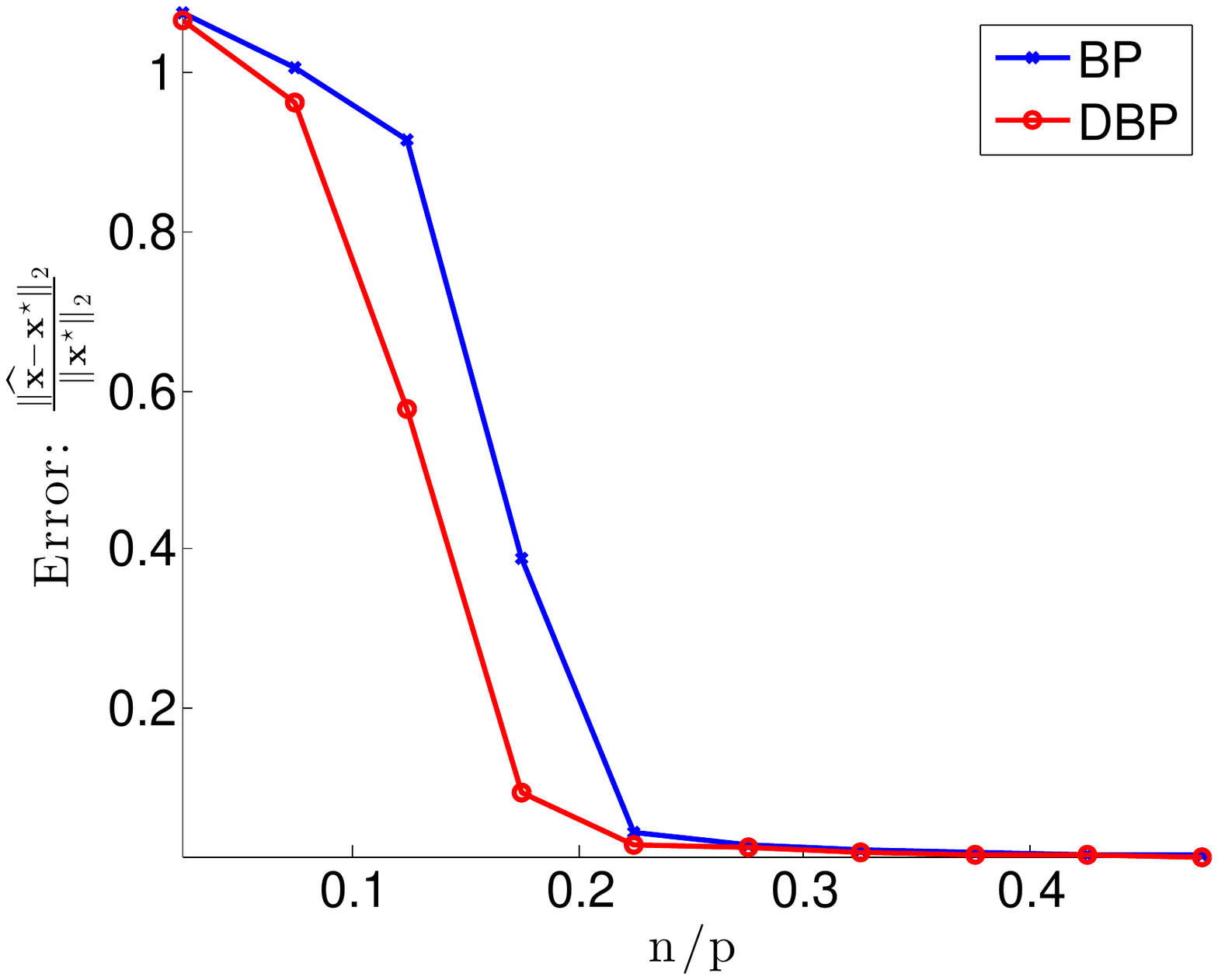}
\vspace{-70pt}
\caption{\sl Recovery error of BP and DBP}\label{fig:BPvsDBP}
\end{figure}
In this section, we show the impact of convexifying two different simplicity objectives under the same dispersive structural assumptions. Specifically, we consider minimizing the convex envelope of the $\ell_q$-regularized dispersive $\ell_0$-``norm''  \cite{obozinski2012convex} versus its convex envelope without regularization over the unit $\ell_\infty$-ball in Section \ref{sec: sparse_dispersive}. To produce the recovery results in Figure \ref{fig:BPvsDBP},  we generate a train of spikes of equal value for $\x^\natural$ in dimensions $p=200$ with a refractoriness of $\Delta=25$ (cf., Figure \ref{fig: example recon}). We then recover $\x^\natural$ from its compressive measurements $\y = \A \x^\natural + \w$, where the noise $\w$ is also a sparse vector, with $15$ non-zero Gaussian values of variance $\sigma=0.01$ and $\A$ is a random column normalized Gaussian matrix. Since the noise is sparse, we encode the data via $\| \y - \A \x\|_1 \le \|\w\|_1$ using the true $\ell_1$-norm of the noise. We produce the data randomly 20 times and report the averaged results. 

Figure \ref{fig:BPvsDBP} measures the relative recovery error with $\frac{\| \x^\natural  -  \hat{\x} \|_2}{\| \x^\natural \|_2 }$, as we vary the number of compressive measurements. The regularized convexification simply leads  to Basis Pursuit formulation (BP), while the TU convexification results in the addition of a budget constraint $\Bbd^T|\x|\leq \mathds{1}$ to the BP formulation, as described in Section \ref{sec: sparse_dispersive}. We refer to the resulting criteria as Dispersive Basis Pursuit (DBP). Since the DBP criteria uses the fact that $\x^\natural$ lies in the unit $\ell_\infty$-ball, we include this constraint in the BP formulation for fairness. We use an interior point method to obtain high accuracy solutions to each formulation. 

Figure \ref{fig:BPvsDBP} shows that DBP outperforms BP as we vary the number of measurements. Note that the number of measurements needed to achieve a certain error is expected to be lower for DBP than BP, as theoretically characterized in \cite{hegde2009compressive}. Hence, by changing the objective in the convexification, Figure \ref{fig: example recon} reinforces the message that we can lose the tightness in capturing certain structured sparsity models. 
\begin{figure}
\centering
\vspace{-30pt}
\begin{tabular}{c c c} \hspace{-20pt}
\includegraphics[scale=.15]{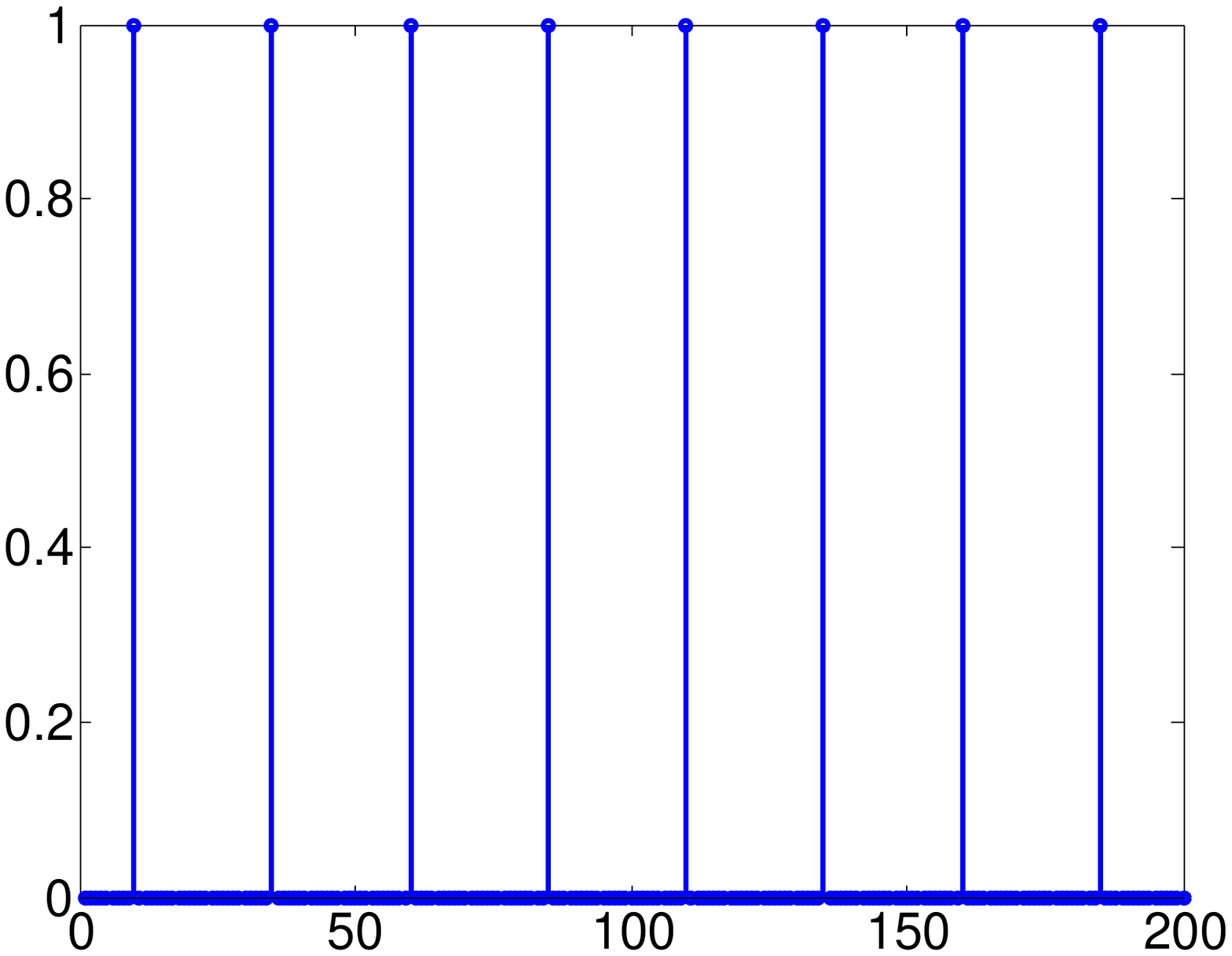} & \hspace{-10mm}
 \includegraphics[scale=.15]{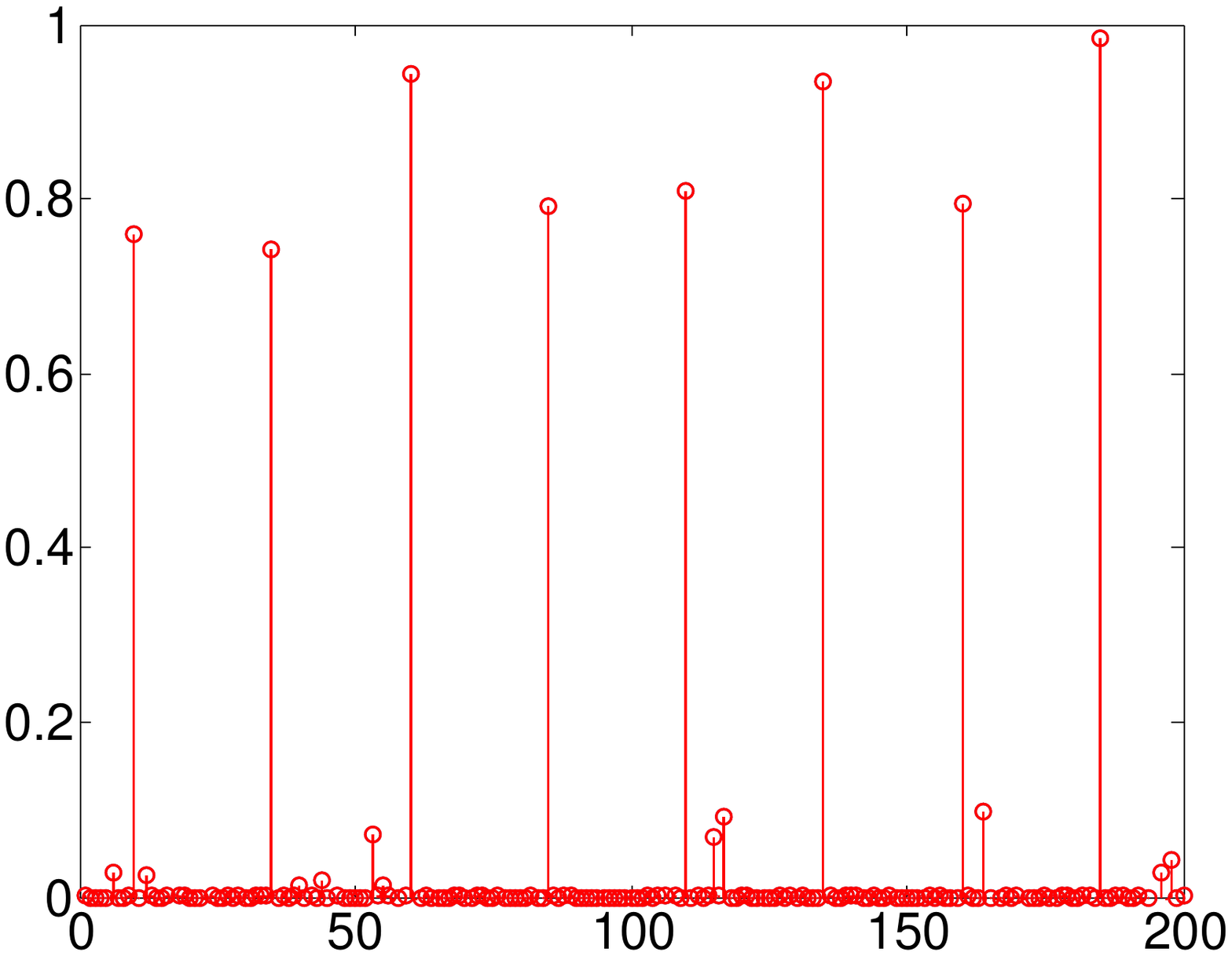} &\hspace{-10mm}
  \includegraphics[scale=.15]{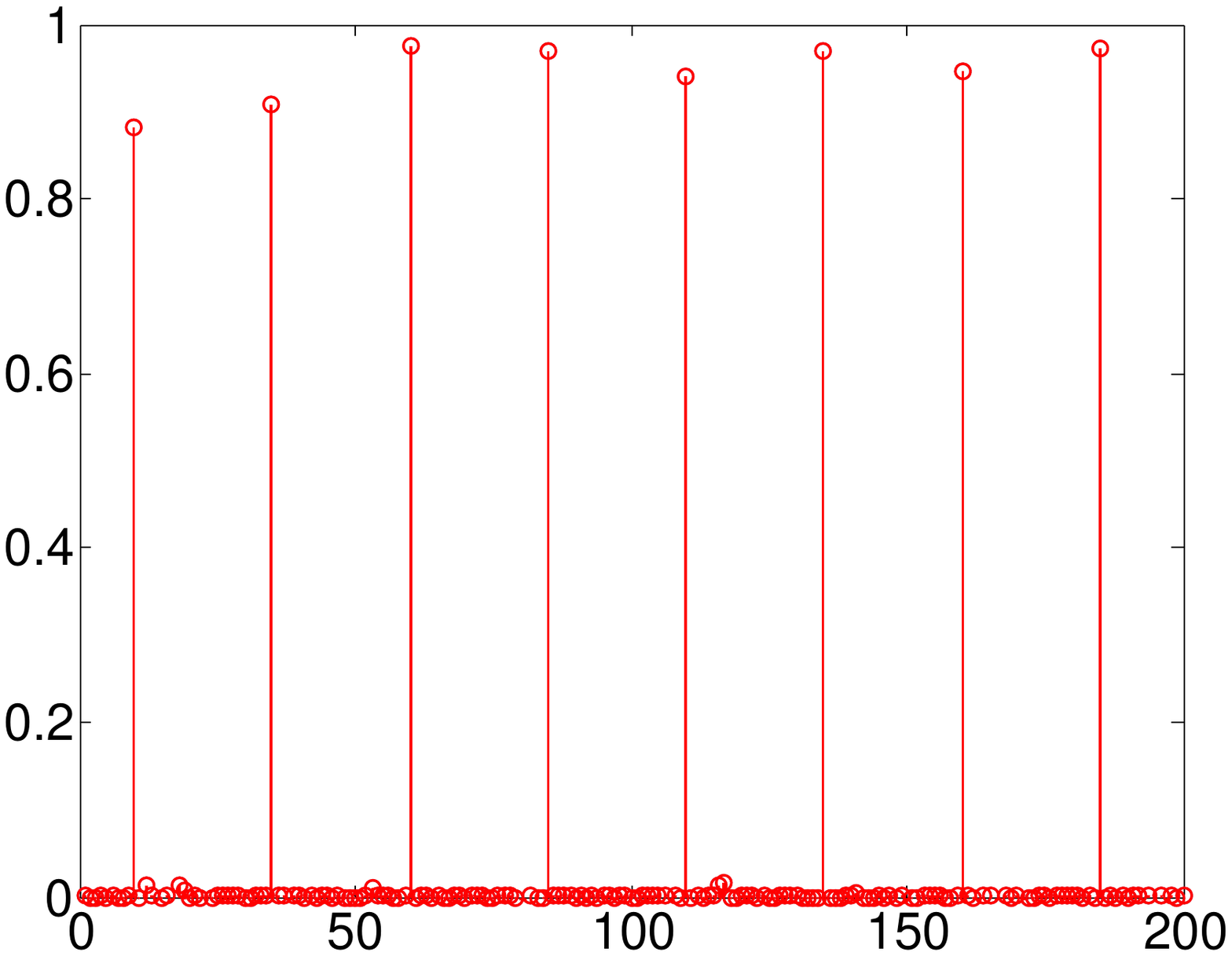}\\[-10mm]
  \hspace{-6mm} {\tiny $\x^\natural$} & \hspace{-6mm}{\tiny $\x_{\text{BP}}$ solution} & \hspace{-6mm}{\tiny $\x_{\text{DBP} }$ solution}\\
 {\tiny relative errors:} & \hspace{-6mm}{\tiny $\frac{\|\x^\natural - \x_{\text{BP}}\|_2}{\| \x^\natural \|_2} = .200$} & \hspace{-6mm}{\tiny $\frac{\|\x^\natural - \x_{\text{DBP}}\|_2}{\| \x^\natural \|_2} = .067$}\\
\end{tabular}
\caption{\sl Example spike train recovery when $n=0.18p$. The DBP formulation (right) shrinks the competing sparse coefficients within the $\Delta$ intervals, resulting in a better reconstruction overall sampling regimes than BP (middle).}\label{fig: example recon}
\end{figure}

\vspace{-10pt}
\section{Conclusions}
We have provided a principled recipe for designing convex formulations that jointly express models of simplicity and structure for sparse recovery, that promotes clustering or dispersiveness. The main hallmark of our approach is its pithiness in generating the prevalent convex structured sparse formulations and in explaining their tightness. Our key idea relies on expressing sparsity structures via simple linear inequalities over the support of the unknown parameters and their corresponding latent group indicators. By recognizing the totally unimodularity of the underlying constraint matrices, we can tractably compute the biconjugation of the corresponding combinatorial simplicity objective subject to structure, and perform tractable recovery using standard optimization techniques.  

\vspace{-25pt}
~~
\subsubsection*{Acknowledgements}

This work was supported in part by the European Commission under Grant MIRG-268398, ERC Future Proof, SNF 200021- 132548, SNF 200021-146750 and SNF CRSII2-147633.

\appendix
\section{Numerical illustration of Sparse $G$-group cover's performance}

\begin{figure}
\centering
\vspace{-10pt}
\includegraphics[clip=true,trim=0cm 6cm 0cm 7cm, scale=.3]{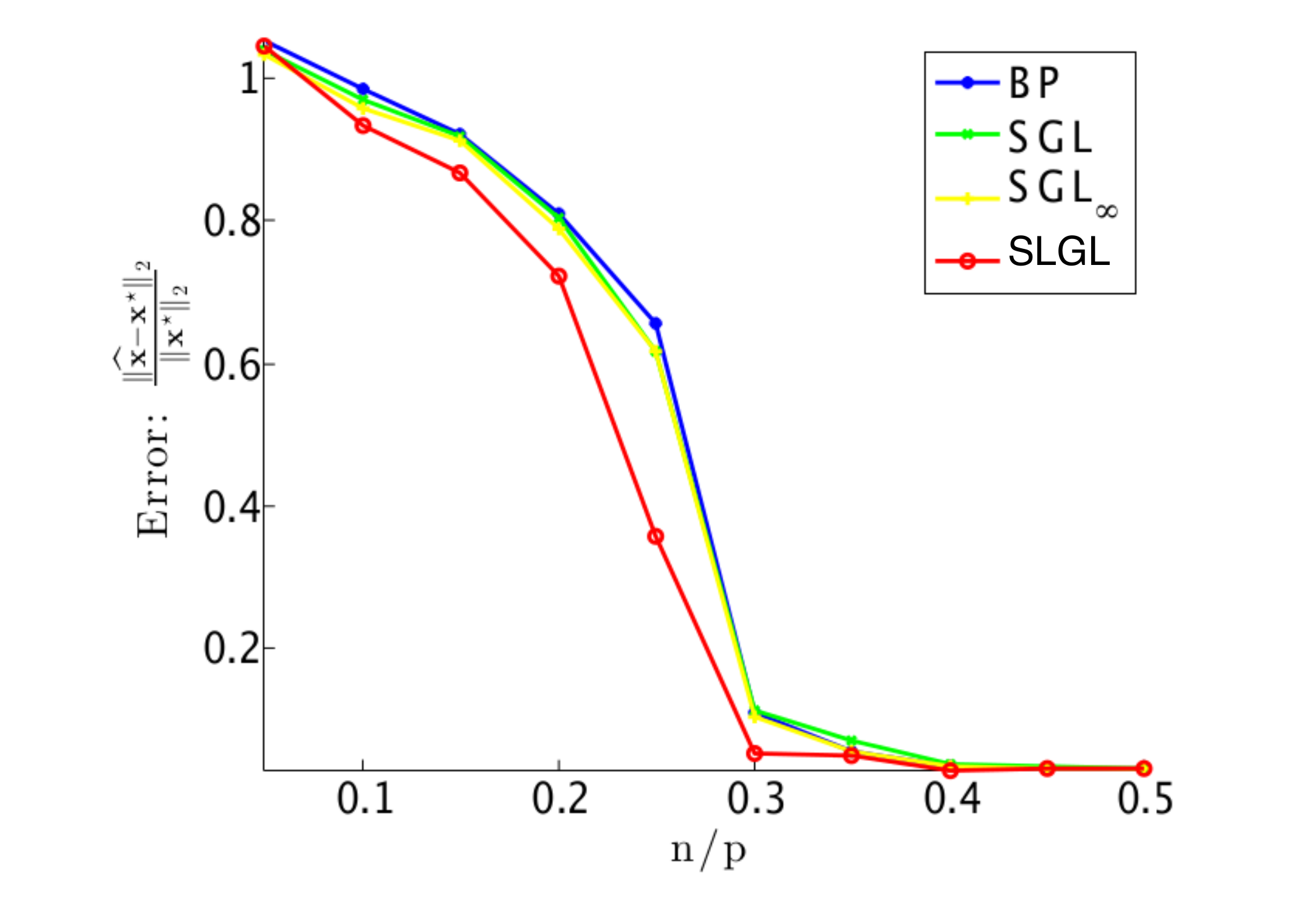}
\vspace{-10pt}
\caption{\label{fig:SLGLvsSGL} Recovery error of SLGL, SGL, and BP}
\end{figure}

In this section, we compare the performance of minimizing the TU relaxation $g_{\GG,G}^{\ast \ast}$ of the proposed Sparse $G$-group cover (c.f., Section \ref{sec:Sarse_grpCover}) in problem \eqref{eq: proto}, which we will call Sparse latent group lasso (SLGL), with Basis pursuit (BP) and Sparse group Lasso (SGL).
Recall the SGL criteria is $\small (1-\alpha)\sum_{\G \in \GG} \sqrt{|\G|} \| \x_\G\|_q + \alpha \| \x_\G \|_1$, with $q=2$ in \cite{simon2013sparse}. We compare also against $SGL_\infty$ where we set $q=\infty$, which is better suited for signals with equal valued non-zero coefficients. 
We generate a sparse signal $\x^\natural$ in dimensions $p=200$, covered by $G=5$ groups, randomly chosen from the $M=29$ groups. The groups generated are interval groups, of equal size of $10$ coefficients, and with an overlap of $3$ coefficients between each two consecutive groups. The true signal $\x^\natural$ has $3$ non-zero coefficients (all set to one) in each of its $5$ active groups (cf., Figure \ref{fig: example recon}). Note that these groups lead a TU group structure $\GG$, so the TU relaxation in this case is tight. We recover $\x^\natural$ from its compressive measurements $\y = \A \x^\natural + \w$, where the noise $\w$ is a random Gaussian vector of variance $\sigma=0.01$ and $\A$ is a random column normalized Gaussian matrix. We encode the data via $\| \y - \A \x\|_2 \le \|\w\|_2$ using the true $\ell_2$-norm of the noise. We produce the data randomly 10 times and report the averaged results. 

Figure \ref{fig:SLGLvsSGL} measures the relative recovery error with $\frac{\| \x^\natural  -  \hat{\x} \|_2}{\| \x^\natural \|_2 }$, as we vary the number of compressive measurements. Since the SLGL criteria uses the fact that $\x^\natural$ lies in the unit $\ell_\infty$-ball, we include this constraint in the all the other formulations for fairness. Since the true signal exhibit strong overall sparsity we use $\alpha=0.95$ in SGL as suggested in \cite{simon2013sparse} (we tried several values of $\alpha$, and this seemed to give the best results for SGL). We use an interior point method to obtain high accuracy solutions to each formulation. Figure \ref{fig:BPvsDBP} shows that SLGL outperforms the other criterias as we vary the number of measurements.

\begin{figure*}
\centering
\begin{tabular}{c c c c c} \hspace{-20pt}
\includegraphics[scale=.13]{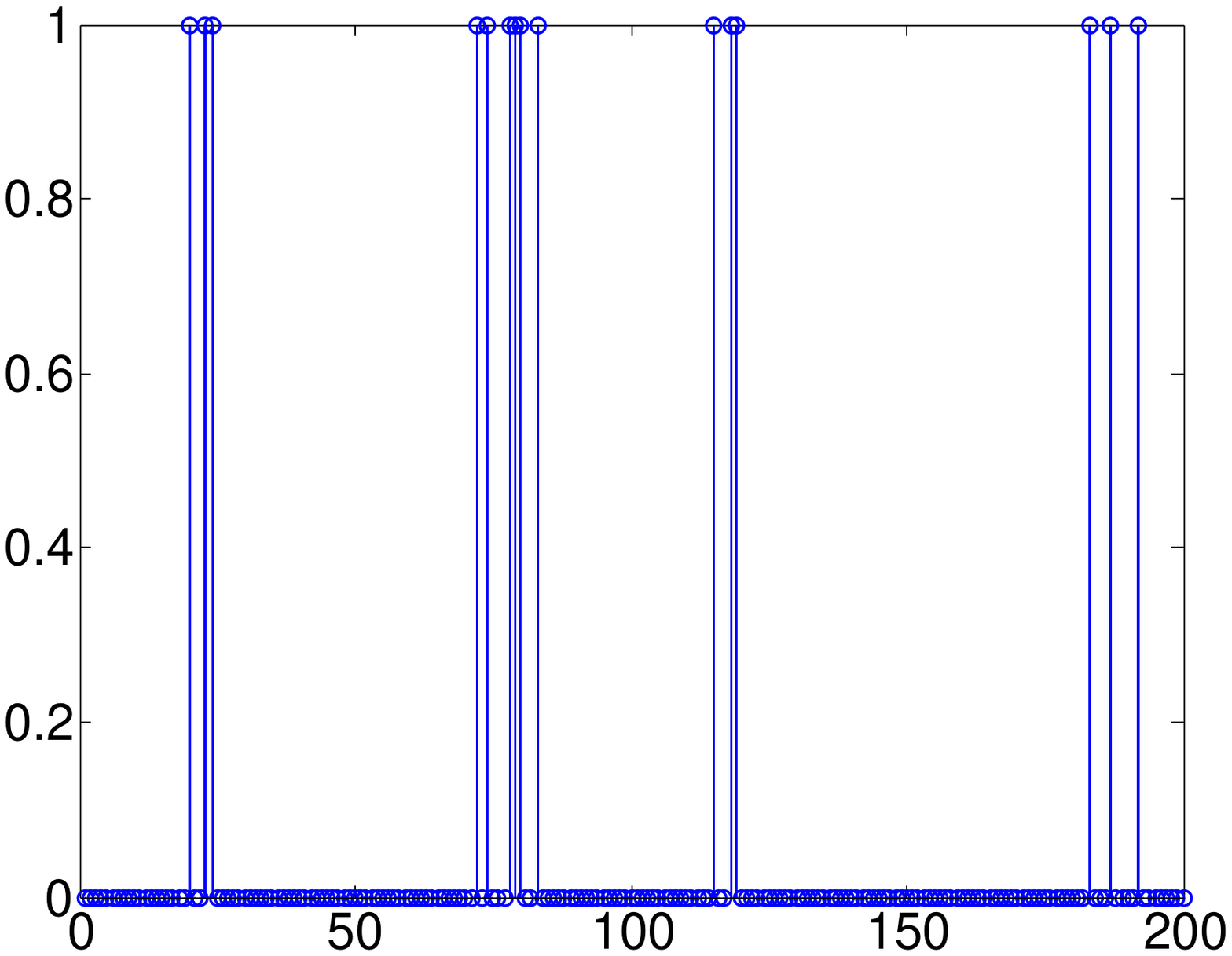} & \hspace{-10mm}
 \includegraphics[scale=.13]{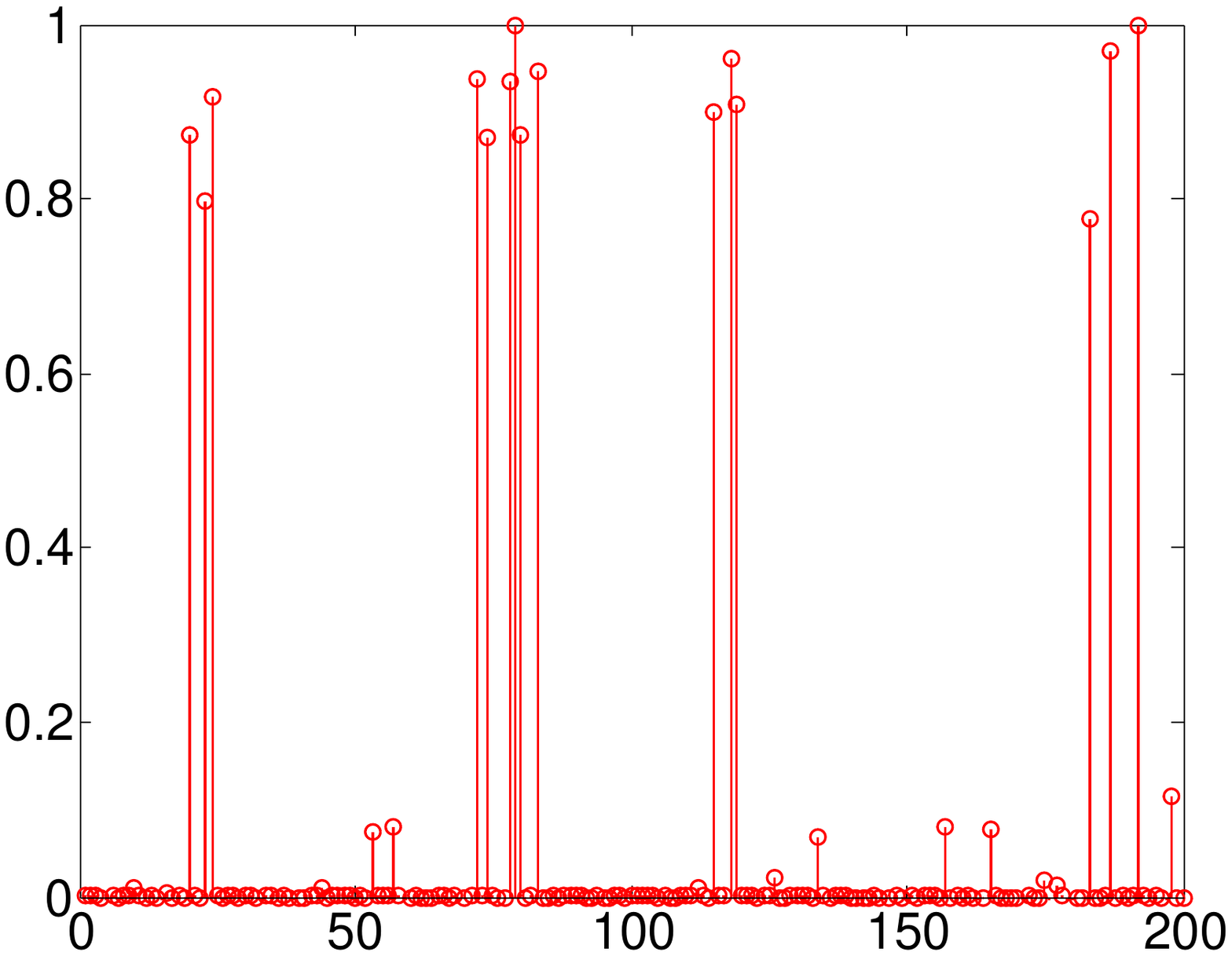} &\hspace{-10mm}
  \includegraphics[scale=.13]{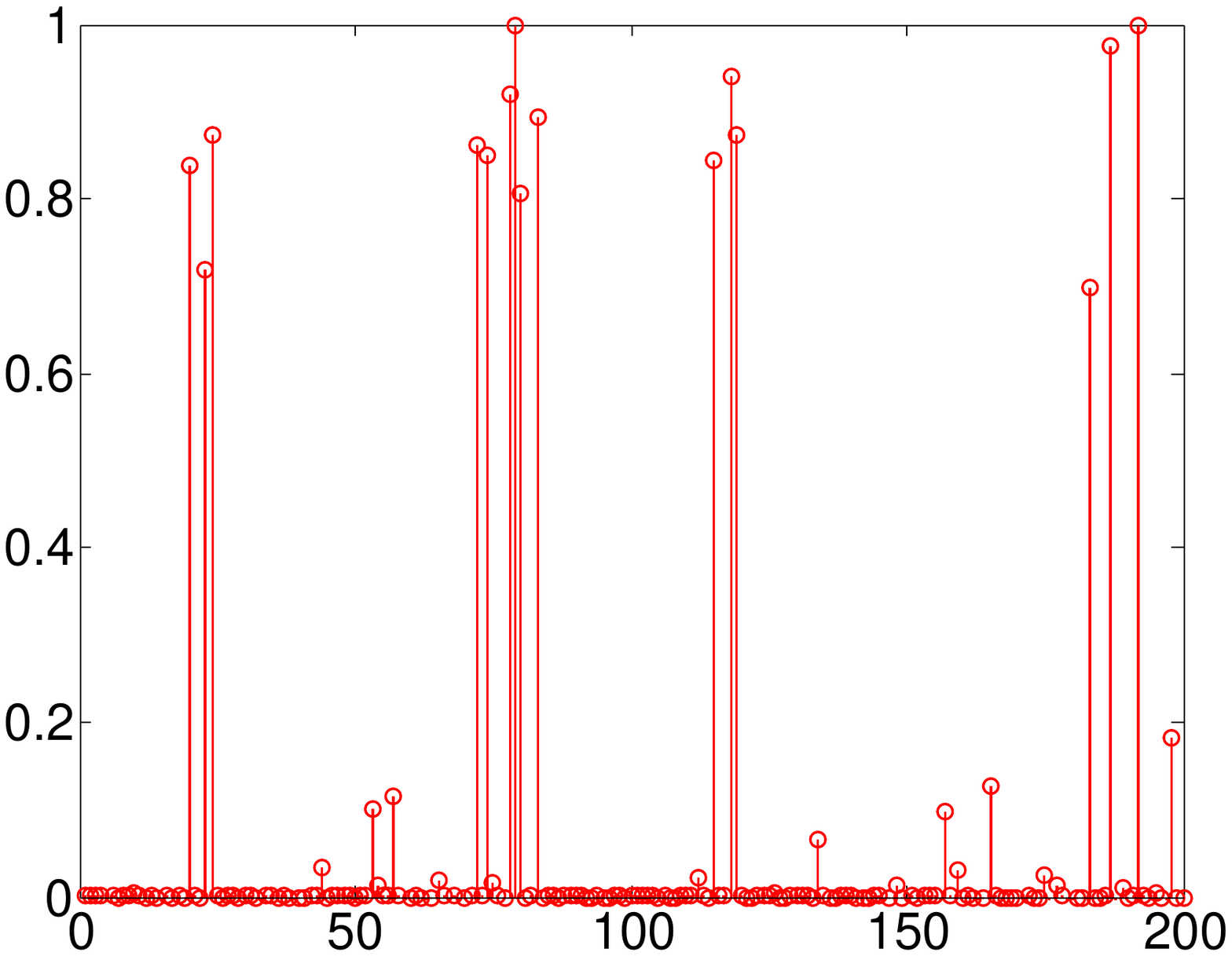} &\hspace{-10mm}
    \includegraphics[scale=.13]{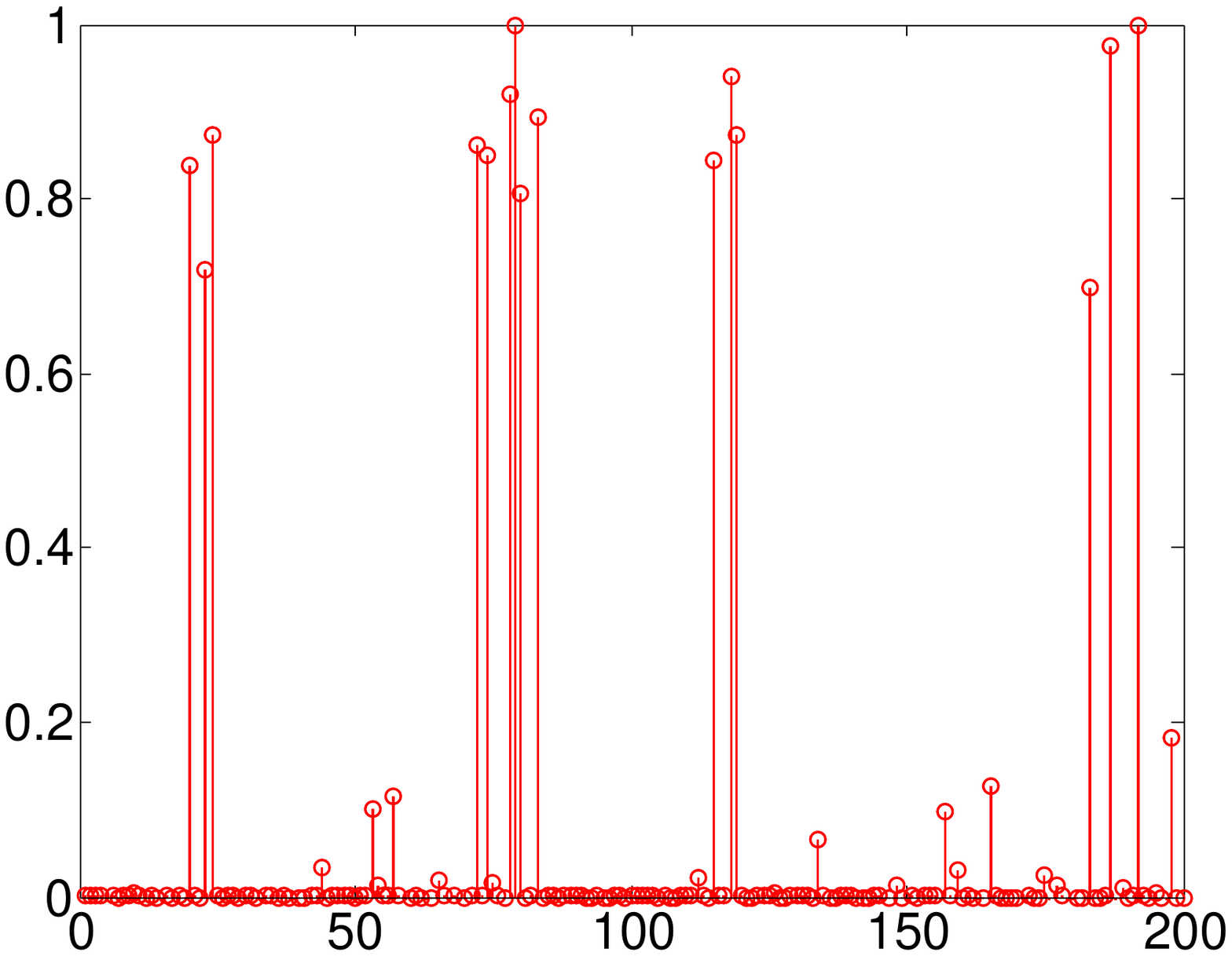} &\hspace{-10mm}
  \includegraphics[scale=.13]{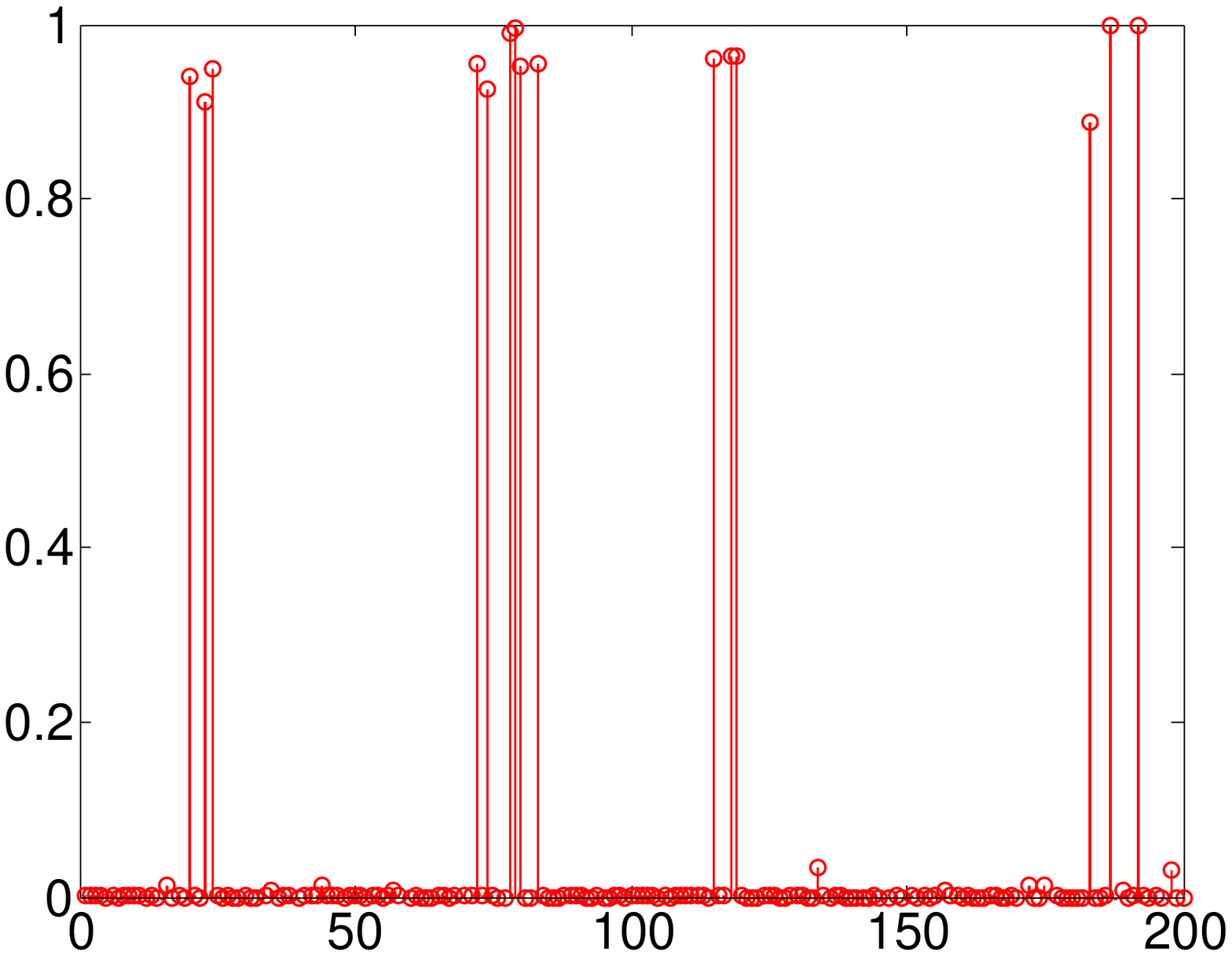}\\[-10mm]
  \hspace{-6mm} {\tiny $\x^\natural$} & \hspace{-6mm}{\tiny $\x_{\text{BP}}$ solution}& \hspace{-6mm}{\tiny $\x_{\text{SGL}}$ solution} & \hspace{-6mm}{\tiny $\x_{\text{SGL}_\infty}$ solution}& \hspace{-6mm}{\tiny $\x_{\text{SLGL} }$ solution}\\
  
\hspace{-5mm} {\tiny relative errors:} & \hspace{-5mm}{\tiny $\frac{\|\x^\natural - \x_{\text{BP}}\|_2}{\| \x^\natural \|_2} = .128$} & \hspace{-3mm}{\tiny $\frac{\|\x^\natural - \x_{\text{SGL}}\|_2}{\| \x^\natural \|_2} = .181$} &\hspace{-3mm}{\tiny $\frac{\|\x^\natural - \x_{\text{SGL}_\infty}\|_2}{\| \x^\natural \|_2} = .085$} & \hspace{-3mm}{\tiny $\frac{\|\x^\natural - \x_{\text{SLGL}}\|_2}{\| \x^\natural \|_2} = .058$}\\
\end{tabular}
\caption{\sl Recovery for $n=0.25p, s=15, p=200, G=5$ out of $M=29$ groups.}\label{fig: example recon}
\end{figure*}

\section{Proof of Proposition \ref{prop:conv_grpInter}}
\begin{prop}[Convexification]
The convex envelope of $g_{\GG, \cap}(\x)$ over the unit $\ell_\infty$-ball is
$$g_{\GG, \cap}^{\ast \ast}(\x) = \begin{cases}  \sum_{\G_i \in \GG} d_i \| \x_{\G_i} \|_\infty  &\text{if $\x \in [-1,1]^p$}\\
\infty &\text{otherwise}
\end{cases}$$
\end{prop}
\begin{proof}
Since $g_{\GG, \cap}(\x)$ is a TU-penalty, we can use Proposition \ref{eq:conv_TU} in the main text, to compute its convex envelope:
\begin{align*}
g_{\GG, \cap}^{\ast \ast}(\x)&= \min_{\s \in [0,1]^p, \omegabf \in [0,1]^M} \{ \db^T \omegabf : \Hb \betabf  \leq 0, |\x| \leq \s \} \\
&= \min_{ \omegabf \in [0,1]^M} \{ \db^T \omegabf : \Hb \colvec{ \omegabf \\ |\x| } \leq 0\} \\
&= \sum_{\G_i \in \GG} d_i \| \x_{\G_i} \|_\infty 
\tag{since ${w_i^\ast} = \| \x_\G \|_\infty$}
\end{align*}
for $\x \in [-1,1]^p$, $g_{\GG, \cap}^{\ast \ast}(\x)=\infty$ otherwise.
\end{proof}

\section{Proof of Proposition \ref{prop:conv_grpCover}}
\begin{prop}[Convexification]
When the group structure leads to a TU biadjacency matrix $ \Bbd$, the convex envelope of the group $\ell_0$-``norm'' over the unit $\ell_\infty$-ball is
\[\small g_{\GG,0}^{\ast \ast}(\x) = \begin{cases} \min_{\omegabf \in [0,1]^M} \{ \db^T \omegabf :  \Bbd \omegabf \geq |\x| \} &\text{if $\x \in [-1,1]^p$} \\
\infty &\text{ otherwise}
\end{cases}\]
\end{prop}
\begin{proof}
Note that $g_{\GG,0}(\x)$ can be written in the form given in Definition \ref{def:TU_penalty} with $\M = [-\Bbd, \I_p]$ and $\cb = 0$. Thus, when $\Bbd$ is TU, so is $\M$ \cite[Proposition 2.1]{nemhauser1999integer}, and thus we can use Proposition \ref{eq:conv_TU}  in the main text, to compute its convex envelope:
\begin{align*}
 g_{\GG,0}^{\ast \ast}(\x) &= \min_{\scriptsize \begin{array}{c}
\s \in [0,1]^{p} \\  \omegabf \in [0,1]^M
\end{array} }    \{  \db^T\omegabf : \Bbd \omegabf \geq \s, |\x| \leq \s \}  \\
&=\min_{ \omegabf \in [0,1]^M }  \{  \db^T\omegabf : \Bbd\omegabf \geq |\x| \}  
\end{align*}
for $\x \in [-1,1]^p$, $ g_{\GG,0}^{\ast \ast}(\x)=\infty$ otherwise.
\vspace{-10pt}
\end{proof}

\section{Proof of Proposition \ref{prop:SGL_notTU}}

\begin{prop}
Given any group structure $\GG$, $g_{\GG, s}(\x)$ is not a TU penalty.
\end{prop}
\begin{proof}
Let $G(\GG \cup \Pt, \Et)$ denote the bipartite graph representation of the group structure $\GG$. We use the linearization trick employed in \cite{kaminski2008quadratic} to reduce $g_{\GG, s}(\x)$ to an integer program. For conciseness, we consider $g_{\GG, s}(\s)$ only for binary vectors $\s \in \{ 0,1 \}^p$, since $g_{\GG, s}(\x)=g_{\GG, s}(\1_{\supp(\x)})$.
\begin{align*} \small
g_{\GG, s}(\s) &=  \min_{\omegabf \in \{0,1\}^M} \{ \sum_{i=1}^M \omega_i \| \s_{\G_i}\|_0: \M \betabf \leq 0\} \\
&= \min_{\omegabf \in \{0,1\}^M} \{ \sum_{(i,j) \in \Et} \omega_i s_j: \M \betabf \leq 0\} \\
&= \min_{\scriptsize \colcst{\omegabf \in \{0,1\}^M \\ \z \in \{0,1\}^{|\Et|}}} \{ \sum_{(i,j) \in \Et} z_{ij}: \M \betabf \leq 0, \E\betabf \leq \z + \mathds{1}\} 
\end{align*}

Recall that $\E$ is the edge-node incidence matrix of $G(\GG \cup \Pt, \Et)$. The constraint $\E\betabf \leq \z -\mathds{1}$ corresponds to $z_{ij} \geq \omega_i + s_j -1, \forall (i,j) \in \Et$. Although both matrices $\M$ and $\E$ are TU, their concatenation $\widetilde{\M}=\colvec{\M \\ \E}$ is not TU. To see this, let us first focus on the case where $\M=[-\Bbd, \I_p]$.

Given any coefficient $i \in \Pt$ covered by at least one group $\G_i$, we denote the corresponding edge in the bipartite graph by $e_j=(i,M+i)$, which corresponds to the $j^{th}$ row of $\E$. This translates into having the entries $\widetilde{\M}_{i,i}=-1,\widetilde{\M}_{i,M+i}=1, \widetilde{\M}_{p+j,i}=1,$ and $\widetilde{\M}_{p+j,M+i}=1 $. The determinant of the submatrix resulting from these entries is $-2$, which contradicts the definition of TU (cf., Def.\ \ref{def:TU_penalty}). It follows then that $\widetilde{\M}$ is TU iff $\GG=\{ \emptyset\}$.

A similar argument holds for $\M=\Hb$. 
\end{proof}

\section{Proof of Proposition \ref{prop:conv_grpLasso_sparse}}
\begin{prop}[Convexification] 
The convex surrogate via Proposition \ref{eq:conv_TU} in the main text, for $g_{\GG, s}(\x)$ with $\M= \Hb$ (i.e., the group intersection model with sparse groups) is given by 
$$\Omega_{\GG, s}(\x) := \sum_{(i,j) \in \Et} (\| \x_{\G_i} \|_\infty + |x_j| -1)_+ $$
for $\x \in [-1,1]^p$, and, $\Omega_{\GG, s}(\x) := \infty$ otherwise. Note that $\Omega_{\GG, s}(\x)\le g_{\GG, s}^{**}(\x)$.
\end{prop}
\begin{proof} For $\x \in [-1,1]^p$,
\small
\begin{align*}
\Omega_{\GG, s}(\x) &= \min_{\scriptsize \colcst{\omegabf \in [0,1]^M \\ \z \in [0,1]^{|\Et|}}} \{\sum_{(i,j) \in \Et}   z_{ij}: \Hb \betabf \leq 0, \E\betabf \leq \z + \mathds{1}, |\x| \leq \s\} \\
&= \sum_{(i,j) \in \Et} (\| \x_{\G_i} \|_\infty + |x_j| -1)_+ 
\end{align*}
since ${\omega}_i^\ast=\| \x_{\G_i} \|_\infty, {\s}^\ast=|\x|,$ and ${z}_{ij}^\ast = ({\omega}_i^\ast + {s}_j^\ast -1)_+$.
\end{proof}

\section{Proof of Proposition \ref{prop:conv_grpLasso_sparse2}}

\begin{prop}[Convexification]
The convex surrogate given by Proposition \ref{eq:conv_TU}  in the main text, for $g_{\GG, s}(\x)$ with $\M = [-\Bbd, \I_p]$ (i.e., the group $\ell_0$-``norm'' with sparse groups) is given by
$$\Omega_{\GG, s}(\x) :=  \min_{\omegabf \in [0,1]^M} \{\sum_{(i,j) \in \Et} (\omega_i + |x_j| -1)_+ : \Bbd \omegabf \geq |\x| \}$$
for $\x \in [-1,1]^p$, $\Omega_{\GG, s}(\x)=\infty$ otherwise.
\end{prop}
\begin{proof} For $\x \in [-1,1]^p$,
\small
\begin{align*}
\Omega_{\GG, s}(\x) &= \hspace{-10pt} \min_{\scriptsize \colcst{\omegabf \in [0,1]^M \\ \z \in [0,1]^{|\Et|}}} \hspace{-5pt} \{  \sum_{(i,j) \in \Et}  \hspace{-5pt} z_{ij}: \Bbd \omegabf \geq \s, \E\betabf \leq \z + \mathds{1}, |\x| \leq \s\} \\
&= \min_{\omegabf \in [0,1]^M} \{\sum_{(i,j) \in \Et} (\omega_i + |x_j| -1)_+ : \Bbd \omegabf \geq |\x| \}
\end{align*}
since ${\s^\ast}=|\x|,$ and ${z}_{ij}^\ast = ({\omega}_i + {s}_j^\ast -1)_+$.
\end{proof}

\section{Proof of Proposition \ref{prop:conv_tree}}

\begin{prop}(Convexification) 
The convexification of the tree $\ell_0$-``norm'' over the unit $\ell_\infty$-ball is given by $$g^{\ast \ast}_{T,0}(\x)= \begin{cases}
\sum_{\G \in \GG_H} \| x_\G \|_\infty &\text{if $\x \in [-1,1]^p$} \\
 \infty &\text{otherwise}
\end{cases}$$
\end{prop}
\begin{proof}
Since this is a TU-penalty we can use Proposition \ref{eq:conv_TU}  in the main text, to compute its convex envelope:
\begin{align*}
g^{\ast \ast}_{T,0}(\x) &=  \min_{\s \in [ 0,1]^p} \{ \mathds{1}^T\s : \Tb \s \geq 0, |\x |\leq \s\}  \\
&\stackrel{\star}{=} \sum_{\G \in \GG_H} \| x_{\G} \|_\infty
\end{align*}
for $\x \in [-1,1]^p$, $\infty$ otherwise, and where the groups $\G \in \GG_{H}$ are defined as each node and all its descendants.
$(\star)$ holds since any feasible $\s$ should satisfy $\s \geq {|\x|}$ and $s_{\text{parent}} \geq s_{\text{child}}$, so starting from the leaves, each leaf satisfies $s_i \geq |x_i|$, and since we are looking to minimize the sum of $s_i$'s, we simply set $s_i=x_i$. For a node $i$ with two children $j,k$ as leaves, it will satisfy $s_i \geq |x_i|, |s_j|, |s_k|$, thus $s_i = \max\{ |x_i|, |x_j|, |x_k|\}$, and so on. Thus, $s_i= \max_{ \{k \text{ is a descendant of $i$ or $i$ itself} \}}{ |x_k|}$
\end{proof}

\section{Proof of Proposition \ref{prop:conv_disp}}
\begin{prop}[Convexification] 
The convex envelope of $g_{\Db}(\x) $ over the unit $\ell_\infty$-ball when $\Bbd^T$ is a TU matrix is given by
$$  g_{\Db}^{\ast \ast}(\x) = \begin{cases} \max_{\G \in \GG} \| \x_\G\|_1 &\text{if  $\x \in [-1,1]^p, \Bbd^T |\x|\leq \mathds{1}$ } \\
\infty &\text{otherwise} \end{cases}$$
\end{prop}
\begin{proof}  Since this is a TU penalty we can use Proposition \ref{eq:conv_TU} in the main text, to compute its convex envelope:
\begin{align*}
 g_{\Db}^{\ast \ast}(\x) &= \small \begin{cases} \min_{\scriptsize\colcst{\omega \in [0,1] \\ \s \in [0,1]^p } } \{ \omega: \Bbd^T \s \leq \omega \mathds{1}, |\x| \leq \s\} &\text{\small if  $\x$  feasible } \\
\infty &\text{\small otherwise} 
\end{cases}\\
&= \begin{cases}
\| \Bbd^T |\x| \|_\infty &\text{if  $\x \in [-1,1]^p, \Bbd^T |\x|\leq \mathds{1}$ } \\
\infty &\text{otherwise} 
\end{cases}
\end{align*}
\vspace{-10pt}
\end{proof}

\section{Proof of Proposition \ref{prop:conv_disp_pair}}
\begin{prop}[Convexification] 
The convex envelope of $g_{\G, \D}(\x)$ over the unit $\ell_\infty$-ball is
$$ g_{\G, \D}^{\ast \ast}(\x) = \begin{cases} \sum_{(i,j) \in \Et} (|x_i| + |x_j| -1)_+  &\text{if  $\x \in [-1,1]^p$ } \\
\infty &\text{otherwise} \end{cases}$$
\end{prop}
\begin{proof}
 We use the linearization trick employed in \cite{kaminski2008quadratic} to reduce $g_{\G, \D}(\x)$ to a TU penalty. Let $\s= \mathds{1}_{\supp(\x)}$,
\begin{align*}
g_{\G, \D}(\x) &= \sum_{(i,j) \in \Et} s_i s_j \\
&= \min_{\z \in \{ 0,1\}^{|\Et|}} \{ \sum_{(i,j) \in \Et} z_{ij}: z_{ij} \geq s_i + s_j -1 \} \\ 
&= \min_{\z \in \{ 0,1\}^{|\Et|}} \{ \sum_{(i,j) \in \Et} z_{ij}: \E_G \s \leq \z - \mathds{1} \}
\end{align*} 
Now we can apply Proposition \ref{eq:conv_TU} in the main text, to compute the convex envelope:
\small
\begin{align*}
g_{\G, \D}^{\ast \ast}(\x)&= \min_{\s \in [0,1]^p, \z \in [ 0,1]^{|\Et|}} \{ \sum_{(i,j) \in \Et} z_{ij}: \E_\G \s \leq \z - \mathds{1} , |\x| \leq \s\}\\
&=  \sum_{(i,j) \in \Et} (|x_i| + |x_j| -1)_+ \tag{${\s}^\ast=\x, {z}_{ij}^\ast = ({s}_i^\ast + {s}_j^\ast -1)_+$}
\end{align*} 
for $\x \in [-1,1]^p$, $g_{\G, \D}^{\ast \ast}(\x)=\infty$ otherwise.
\end{proof}

\bibliographystyle{plain}
\bibliography{biblio,add_bibtex}

\end{document}